\def\year{2021}\relax
\pdfoutput=1 %for submit to arXiv? 
\documentclass[letterpaper]{article} % DO NOT CHANGE THIS
\usepackage{aaai21}  % DO NOT CHANGE THIS
\usepackage{times}  % DO NOT CHANGE THIS
\usepackage{helvet} % DO NOT CHANGE THIS
\usepackage{courier}  % DO NOT CHANGE THIS
\usepackage[hyphens]{url}  % DO NOT CHANGE THIS
\usepackage{graphicx} % DO NOT CHANGE THIS
\usepackage{natbib}  % DO NOT CHANGE THIS AND DO NOT ADD ANY OPTIONS TO IT
\usepackage{caption} % DO NOT CHANGE THIS AND DO NOT ADD ANY OPTIONS TO IT
\usepackage{amsfonts,amsmath,amssymb,amsthm}
\DeclareMathOperator*{\argmax}{arg\,max}
\DeclareMathOperator*{\argmin}{arg\,min}

\DeclareMathOperator*{\arginf}{arg\,inf}

\usepackage{subcaption}
\usepackage{hyperref}

\theoremstyle{plain}
\newtheorem{thm}{Theorem}
\newtheorem{lem}{Lemma}
\newtheorem{prop}{Proposition}

\theoremstyle{definition}
\newtheorem{defn}{Definition}

\usepackage[linesnumbered,ruled]{algorithm2e}

\SetCommentSty{mycommfont}
\SetKwInput{KwInput}{Input}                % Set the Input
\SetKwInput{KwOutput}{Output}              % set the Output
\SetKwInput{KwInitialization}{Initialization}
\SetKwInput{KwRequire}{Require}

\usepackage[switch]{lineno}  %
\usepackage{adjustbox}
\usepackage{tabularx}

\title{Distributional Reinforcement Learning via Moment Matching}
\author{
Thanh Tang Nguyen, Sunil Gupta, Svetha Venkatesh 
}
\affiliations{
    Applied Artificial Intelligence Institute (A$^2$I$^2$), Deakin University, Australia \\
    \{thanhnt,sunil.gupta,svetha.venkatesh\}@deakin.edu.au
}
\begin{document}
% \linenumbers  %

\maketitle

\begin{abstract}
We consider the problem of learning a set of probability distributions from the empirical Bellman dynamics in distributional reinforcement learning (RL), a class of state-of-the-art methods that estimate the distribution, as opposed to only the expectation, of the total return. We formulate a method that learns a finite set of statistics from each return distribution via neural networks, as in \cite{DBLP:conf/icml/BellemareDM17,DBLP:conf/aaai/DabneyRBM18}. Existing distributional RL methods however constrain the learned statistics to \textit{predefined} functional forms of the return distribution which is both restrictive in representation and difficult in maintaining the predefined statistics. Instead, we learn \textit{unrestricted} statistics, i.e., deterministic (pseudo-)samples, of the return distribution by leveraging a technique from hypothesis testing known as maximum mean discrepancy (MMD), which leads to a simpler objective amenable to backpropagation.  Our method can be interpreted as implicitly matching all orders of moments between a return distribution and its Bellman target. We establish sufficient conditions for the contraction of the distributional Bellman operator and provide finite-sample analysis for the deterministic samples in distribution approximation. Experiments on the suite of Atari games show that our method outperforms the distributional RL baselines and sets a new record in the Atari games for non-distributed agents.

% Disttributional reinforcement learning (RL) has achieved state-of-the-art performance in Atari games by recasting the traditional RL into a distribution estimation problem, explicitly estimating the probability distribution instead of the expectation of a total return. The bottleneck in distributional RL lies in the estimation of this distribution where one must resort to an approximate representation of the return distributions which are infinite-dimensional. Most existing methods focus on learning a set of predefined statistic functionals of the return distributions requiring involved projections to maintain the order statistics. We take a different perspective using deterministic sampling wherein we approximate the return distributions with a set of deterministic particles that are not attached to any predefined statistic functional,  allowing us to freely approximate the return distributions. The learning is then interpreted as evolution of these particles so that a distance between the return distribution and its target distribution is minimized. This learning aim is realized via maximum mean discrepancy (MMD) distance which in turn leads to a simpler loss amenable to backpropagation. Experiments on the suite of Atari 2600 games show that our algorithm outperforms the standard distributional RL baselines and sets a new record in the Atari games for non-distributed agents. 
\end{abstract}
\section{Introduction}

A fundamental aspect in reinforcement learning (RL) is the value of an action in a state which is formulated as the expected value of the \textit{return}, i.e., the expected value of the discounted sum of rewards when the agent follows a policy starting in that state and executes that action \cite{sutton1998introduction}.  Learning this expected action-value via Bellman's equation \citep{Bellman1957} is central to value-based RL such as temporal-difference (TD) learning \citep{DBLP:journals/ml/Sutton88}, SARSA \citep{Rummery94on-lineq-learning}, and Q-learning \citep{Watkins92q-learning}. Recently, however, approaches known as \textit{distributional} RL that aim at learning the distribution of the return have shown to be highly effective in practice \citep{DBLP:conf/uai/MorimuraSKHT10,DBLP:conf/icml/MorimuraSKHT10,DBLP:conf/icml/BellemareDM17,DBLP:conf/aaai/DabneyRBM18,DBLP:conf/icml/DabneyOSM18,yang2019fully}.

Despite many algorithmic variants with impressive practical performance \citep{DBLP:conf/icml/BellemareDM17,DBLP:conf/aaai/DabneyRBM18,DBLP:conf/icml/DabneyOSM18,yang2019fully}, they all share the same characteristic that they explicitly learn a set of statistics of \textit{predefined} functional forms to approximate a return distribution. Using predefined statistics can limit the learning due to the statistic constraints it imposes and the difficulty to maintain such predefined statistics. In this paper, we propose to address these limitations by instead learning a set of \textit{unrestricted} statistics, i.e., deterministic (pseudo-)samples, of a return distribution that can be evolved into any functional form. We observe that the deterministic samples can be deterministically learned to simulate a return distribution by utilizing an idea from statistical hypothesis testing known as maximum mean discrepancy (MMD). This novel perspective requires a careful design of algorithm and a further understanding of distributional RL associated with MMD. 

Leveraging this perspective, we are able to provide a novel algorithm to eschew the predefined statistic limitations in distributional RL and give theoretical understanding of distributional RL within this perspective. Our approach is also conceptually amenable for natural extension along the lines of recent modelling improvements to distributional RL brought by Implicit Quantile Networks (IQN) \citep{DBLP:conf/icml/DabneyOSM18}, and Fully parameterized Quantile Function (FQF) \citep{yang2019fully}.
% Our approach is also orthogonal to the recent modelling improvements to distributional RL that an extension from such modelling improvements to further improve modelling capacity in our framework is conceptually natural. 
Specifically, our key contributions are 

\begin{enumerate}
    \item We provide a novel approach to distributional RL using deterministic samples via MMD that addresses the limitations in the existing distributional RL that use predefined statistics;

    \item We provide theoretical understanding of distributional RL within our framework, specifically the contraction property of the distributional Bellman operator and the non-asymptotic convergence of the approximate distribution from deterministic samples; 
    
    \item We demonstrate the practical effectiveness of our framework in both tabular RL and large-scale experiments where 
    our method outperforms the standard distributional RL methods and even establishes a new record in the Atari games for non-distributed agents. 
    % we can achieve a new state of the art with a simple modelling choice. 
\end{enumerate}

\noindent \textbf{Outline}. After carefully reviewing relevant background and related works, and discussing their limitations, we present our novel algorithmic approach to address these issues followed by theoretical analysis. We then present the experiments to confirm the effectiveness of our approach empirically and conclude our work. 

% The paper is organized as follows. We first carefully review relevant background material and related works, and their limitations. We then propose to overcome the limitations by providing a novel algorithm followed by theoretical analysis. The experiments for tabular case and large-scale Atari games confirm the effectiveness of our approach empirically which are followed by a conclusion.

\section{Background and Related Works}

\subsection{Expected RL}
In a standard RL setting, an agent interacts with an environment via a Markov Decision Process $(\mathcal{S}, \mathcal{A}, R, P, \gamma)$ \citep{puterman2014markov} where $\mathcal{S}$ and $\mathcal{A}$ denote state and action spaces, resp., $\mathcal{R}$ the reward measure, $P(\cdot| s,a)$ the transition kernel measure, and $\gamma \in [0,1)$ a discount factor. A policy $\pi(\cdot|s)$ maps a state to a distribution over the action space. 

Given a policy $\pi$, the discounted sum of future rewards following policy $\pi$ is the random variable 
\begin{align}
    Z^{\pi}(s,a) = \sum_{t=0}^{\infty} \gamma^{t} R(s_t, a_t),
    \label{eq:return_rv}
\end{align}
where $s_0 = s, a_0 = a, s_t \sim P(\cdot | s_{t-1}, a_{t-1})$, $a_t \sim \pi(\cdot | s_t)$, and $R(s_t,a_t) \sim \mathcal{R}(\cdot|s_t, a_t)$. The goal in expected RL is to find an optimal policy $\pi^*$ that maximizes the action-value function $Q^{\pi}(s,a) := \mathbb{E}[Z^{\pi}(s,a)]$. A common approach is to find the unique fixed point $Q^{*} = Q^{\pi^{*}}$ of the Bellman optimality operator \citep{Bellman1957} $T: \mathbb{R}^{\mathcal{S} \times \mathcal{A}} \rightarrow \mathbb{R}^{\mathcal{S} \times \mathcal{A}}$ defined by 
\begin{align*}
    T Q(s,a) := \mathbb{E}[R(s,a)] + \gamma \mathbb{E}_{P} [\max_{a'} Q(s', a')], \forall (s,a).
\end{align*}
A standard approach to this end is Q-learning \citep{Watkins92q-learning} which maintains an estimate $Q_{\theta}$ of the optimal action-value function $Q^{*}$ and iteratively improves the estimation via the Bellman backup 
\begin{align*}
    Q_{\theta}(s,a) \leftarrow \mathbb{E}[R(s,a)] + \gamma \mathbb{E}_{P} [\max_{a'} Q_{\theta}(s', a')].
\end{align*}
Deep Q-Network (DQN) \citep{mnih2015human} achieves human-level performance on the Atari benchmark by leveraging a convolutional neural network to represent $Q_{\theta}$ while using a replay buffer and a target network to update $Q_{\theta}$.

\subsection{Additional Notations} 
Let $\mathcal{X} \subseteq \mathbb{R}^d$ be an open set. Let $\mathcal{P}(\mathcal{X})$ be the set of Borel probability measures on $\mathcal{X}$. Let $\mathcal{P}(\mathcal{X})^{\mathcal{S} \times \mathcal{A}}$ be the Cartesian product of $\mathcal{P}(\mathcal{X})$ indexed by $\mathcal{S} \times \mathcal{A}$. For any $\alpha \geq 0$, let $\mathcal{P}_{\alpha}(\mathcal{X}) := \{p \in \mathcal{P}(\mathcal{X}): \int_{\mathcal{X}} \|x\|^{\alpha} p(dx) < \infty  \}$. When $d=1$, let $m_n(p) := \int_{\mathcal{X}} x^n p(dx)$ be the $n$-th order moment of a distribution $p \in \mathcal{P}(\mathcal{X})$, and let  
\begin{align*}
    \mathcal{P}_{*}(\mathcal{X}) = \bigg\{p \in \mathcal{P}(\mathcal{X}):  \limsup_{n \rightarrow \infty} \frac{|m_n(p)|^{1/n}}{n} =0 \bigg\}. 
\end{align*}
Note that if $\mathcal{X}$ is a bounded domain in $\mathbb{R}$, then $\mathcal{P}_{*}(\mathcal{X}) = \mathcal{P}(\mathcal{X})$. Denote by $\delta_z$ the Dirac measure, i.e., the point mass, at $z$. Denote by $\Delta_n$ the $n$-dimensional simplex. 

% The action-value function estimate $Q_{\theta}$ can be represented by e.g. a neural network and trained by minimizing the squared temporal difference (TD) error $\delta_t^2 = \left(r_t + \gamma \max_{a'} Q_{\theta}(s_{t+1}, a') - Q_{\theta}(s_t, a_t) \right)^2$
% \begin{align*}
%     \delta_t^2 = \left(r_t + \gamma \max_{a'} Q_{\theta}(x_{t+1}, a') - Q_{\theta}(x_t, a_t) \right)^2
% \end{align*}
% over a transition sample $(s_t, a_t, r_t, s_{t+1})$ while following $\epsilon$-greedy policy over $Q_{\theta}$.  

\subsection{Distributional RL}
Instead of estimating only the expectation $Q^{\pi}$ of $Z^{\pi}$, distributional RL methods \citep{DBLP:conf/icml/BellemareDM17,DBLP:conf/aaai/DabneyRBM18,DBLP:conf/icml/DabneyOSM18,DBLP:conf/aistats/RowlandBDMT18,yang2019fully} explicitly estimate the return distribution $\mu^{\pi}=\text{law}(Z^{\pi})$ as an auxiliary task. Empirically, this auxiliary task has been shown to significantly improve the performance in the Atari benchmark. Theoretically, in the policy evaluation setting, the distributional version of the Bellman operator is a contraction in the $p$-Wasserstein metric \citep{DBLP:conf/icml/BellemareDM17} and Cr\'amer distance \citep{DBLP:conf/aistats/RowlandBDMT18} (but not in total variation distance \citep{Chung1987DiscountedMD}, Kullback-Leibler divergence and Komogorov-Smirnov distance \citep{DBLP:conf/icml/BellemareDM17}). The contraction implies the uniqueness of the fixed point of the distributional Bellman operator. In control settings with tabular function approximations, distributional RL has a well-behaved asymptotic convergence in Cr\'amer distance when the return distributions are parameterized by categorical distributions \citep{DBLP:conf/aistats/RowlandBDMT18}. \citet{bellemare2019distributional} establish the asymptotic convergence of distributional RL in policy evaluation in linear function approximations. \citet{DBLP:conf/aaai/LyleBC19} examine behavioural differences between distributional RL and  expected RL, aligning the success of the former with non-linear function approximations. 

\subsubsection{Categorical distributional RL (CDRL)}
CDRL \citep{DBLP:conf/icml/BellemareDM17} approximates a distribution $\eta$ by a categorical distribution $\hat{\eta} = \sum_{i=1}^N \theta_i \delta_{z_i}$ where $z_1 \leq z_2 \leq ... \leq z_N $ is a set of fixed supports and $\{\theta_i\}_{i=1}^N$ are learnable probabilities. The learnable probabilities $\{\theta_i\}_{i=1}^N$ are found in such way that $\hat{\eta}$ is a projection of $\eta$ onto $\{ \sum_{i=1}^N p_i \delta_{z_i}: \{p_i\}_{i=1}^N \in \Delta_N \}$ w.r.t. the Cr\'amer distance \citep{DBLP:conf/aistats/RowlandBDMT18}. In practice, C51 \citep{DBLP:conf/icml/BellemareDM17}, an instance of CDRL with $N=51$, has shown to perform favorably in Atari games. 

\subsubsection{Quantile regression distributional RL (QRDRL)}
QRDRL \citep{DBLP:conf/aaai/DabneyRBM18} approximates a distribution $\eta$ by a mixture of Diracs $\hat{\eta} = \frac{1}{N} \sum_{i=1}^N \delta_{\theta_i}$ where $\{\theta_i\}_{i=1}^N$ are learnable in such a way that $\hat{\eta}$ is a projection of $\eta$ on $\{  \frac{1}{N} \sum_{i=1}^N \delta_{z_i}: \{z_i\}_{i=1}^N \in \mathbb{R}^N \}$ w.r.t. to the 1-Wasserstein distance. Consequently, $\theta_i = F_{\eta}^{-1}( \frac{2i-1}{2N} )$ where $F_{\eta}^{-1}$ is the inverse cumulative distribution function of $\eta$. Since the quantile values $\{F_{\eta}^{-1}( \frac{2i-1}{2N})\}$ at the fixed quantiles $\{\frac{2i-1}{2N}\}$ is a minimizer of an asymmetric quantile loss from quantile regression literature (thus the name QRDRL) and the quantile loss is compatible with stochastic gradient descent (SGD), the quantile loss is used for QRDRL in practice. QR-DQN-1 \citep{DBLP:conf/aaai/DabneyRBM18}, an instance of QRDRL with Huber loss, performs favorably empirically in Atari games. 

\subsubsection{Other distributional methods} 
Some recent distributional RL methods have made modelling improvements to QRDRL. Two typical improvements are from Implicit Quantile Networks (IQN) \citep{DBLP:conf/icml/DabneyOSM18}, and Fully parameterized Quantile Function (FQF) \citep{yang2019fully}. IQN uses implicit models to represent the quantile values $\{\theta_i\}$ in QRDRL, i.e., instead of being represented by fixed network outputs, $\{\theta_i\}$ are the outputs of a differentiable function (e.g., neural networks) on the samples from a base sampling distribution (e.g., uniform). FQF further improves IQN by optimizing the locations of the base samples for IQN, instead of using random base samples as in IQN, i.e., both quantiles and quantile values are learnable in FQF. 

% Since these modeling improvements are orthogonal to the predefined statistic principle (presented below) in CDRL and QRDRL, we consider CDRL and QRDRL as the fundamental distributional RL methods which we will directly improve upon in our work. In other words, the modeling improvements from IQN/FQF also straightforwardly apply to our work similar to the way they apply to CDRL/QRDRL. 

% However, the quantitative question of how much distributional RL improves against standard RL, e.g. in sample efficiency, remains open. 

\subsection{Predefined statistic principle}
Formally, a statistic is any functional $\zeta: \mathcal{P}(\mathcal{X}) \rightarrow \mathbb{R}$ that maps a distribution $p \in \mathcal{P}(\mathcal{X})$ to a scalar $\zeta(p)$, e.g., the expectation $\zeta(p) = \int_{\mathcal{X}} x p(dx)$ is a common statistic in RL. Here, we formally refer to a \textit{predefined statistic} as the one whose functional form is specified before the statistic is learned. In contrast, an \textit{unrestricted statistic} does not subscribe to any specific functional form (e.g., the median of a distribution $\eta$ is a predefined statistic as its functional form is predefined via $F^{-1}_{\eta}(\frac{1}{2})$ while any empirical sample $z \sim \eta$ can be considered an unrestricted statistic of $\eta$). 

Though CDRL and QRDRL are two different variants of distributional RL methodology, they share a unifying characteristic that they both explicitly learn a finite set of predefined statistics, i.e., statistics of \textit{predefined} functional forms \citep{DBLP:conf/icml/RowlandDKMBD19}. We refer to this as predefined statistic principle. This is clear for QRDRL as the statistics to be learned about a distribution $\eta$ are  $\{\zeta_1, ..., \zeta_N\}$ where 
\begin{align*}
    \zeta_i(\eta) := F^{-1}_{\eta}(\frac{2i-1}{N}), \forall i \in \{1,...,N\}.
\end{align*}
It is a bit more subtle for CDRL. It can be shown in \citep{DBLP:conf/icml/RowlandDKMBD19} that CDRL is equivalent to learning the statistics $\{\zeta_1, ..., \zeta_{N-1}\}$ where 
\begin{align*}
    \zeta_i(\eta) := \mathbb{E}_{Z \sim \eta} \left[ 1_{\{Z < z_i\}} + 1_{\{z_i \leq Z < z_{i+1}\}} \frac{z_{i+1}-Z}{z_{i+1} -z_i} \right], \forall i. 
\end{align*}

Learning predefined statistics as in CDRL and QRDRL however can suffer two limitations in (i) statistic representation and (ii) difficulty in maintaining the predefined statistics. Regarding (i), given the same fixed budget of $N$ statistics to approximate a return distribution $\eta$, CDRL restrictively associates the statistic budget to $N$ fixed supports $\{z_i\}_{i=1}^N$ while QRDRL constrains the budget to $N$ quantile values at specific quantiles. Instead, the statistic budget should be freely learned into any form as long as it could simulate the target distribution $\eta$ sensibly.  Regarding (ii), the fixed supports in CDRL require a highly involved projection step to be able to use KL divergence as the Bellman backup changes the distribution supports; QRDRL requires that the statistics must satisfy the constraints for valid quantile values at specific quantiles, e.g., the statistics to be learned are order statistics. In fact, QR-DQN  \citep{DBLP:conf/aaai/DabneyRBM18}, a typical instance of QRDRL, implicitly maintains the order statistics via an asymmetric quantile loss but still does not guarantee the monotonicity of the obtained quantile estimates. A further notice regarding (ii) recognized in \citep{DBLP:conf/icml/RowlandDKMBD19} is that since in practice we do not observe the environment dynamic but only samples of it, a naive update to learn the predefined statistics using such samples can collapse the approximate distribution due to the different natures of samples and statistics (in fact, \citep{DBLP:conf/icml/RowlandDKMBD19} proposes imputation strategies to overcome this problem). Instead, the statistics to be learned should be free of all such difficulties to reduce the learning burden. 

One might say that IQN/FQF (discussed in the previous subsection) can help QRDRL overcome these limitations. While the modeling improvements in IQN/FQF are practically effective, IQN/FQF however still embrace the predefined statistic principle above as they built upon QRDRL with an improved modelling capacity. In this work we propose an alternative approach to distributional RL that directly eschews the predefined statistic principle used in the prior distributional RL methods, i.e., the finite set of statistics in our approach can be evolved into any functional form and thus also reduces the need to maintain any statistic constraints. If we informally view improvements to CDRL/QRDRL into two dimensions: either modelling dimension or statistic dimension, IQN/FQF lie in the modelling dimension while our work belongs to the statistic dimension. We notice that this does not necessarily mean one approach is better than the other, but rather two orthogonal approaches where the modelling improvements in IQN/FQF can naturally apply to our work to further improve the modeling capacity. We leave these modeling extensions to the future work and focus the present work only on unrestricted statistics with the simplest modelling choice as possible. 
% We notice that our work is not another modelling improvement as in IQN/FQF

% The modelling improvement in IQN/FQN can straightforwardly apply to our alternative framework. 

% In practical setting, all distributional RL algorithms \citep{DBLP:conf/icml/BellemareDM17,DBLP:conf/aaai/DabneyRBM18,DBLP:conf/icml/DabneyOSM18,DBLP:conf/aistats/RowlandBDMT18,yang2019fully} must approximate return distributions by a finite set of statistics as the space of Borel distributions is infinite-dimensional.

% [What predefined statistics]

% [Why maintain it is difficult: (i) projection, (ii) mess up statistics with samples]. 

\section{Distributional RL via Moment Matching}

\subsection{Maximum mean discrepancy} 
% Given two set of samples $(z_i)_{i=1}^N \sim p$ and $(w_j)_{j=1}^M \sim q$, the two-sample hypothesis testing problem is concerned with the question of determining whether these two sets of samples come from the same distribution, i.e., whether $p = q$ given only their samples. As an approach to the two-sample hypothesis testing problem, maximum mean discrepancy (MMD) \citep{DBLP:journals/jmlr/GrettonBRSS12} is defined via the embedding of probability measures into a reproducing kernel Hilbert space (RKHS). 
Let $\mathcal{F}$ be a reproducing kernel Hilbert space (RKHS) associated with a continuous kernel $k(\cdot, \cdot)$ on $\mathcal{X}$. Consider $p, q \in \mathcal{P}(\mathcal{X})$, and let $Z$ and $W$ be two random variables with distributions $p$ and $q$, respectively. The maximum mean discrepancy (MMD) \citep{DBLP:journals/jmlr/GrettonBRSS12} between $p$ and $q$ is defined as 
\begin{align*}
    &\text{MMD}(p, q; \mathcal{F}) := \sup_{f \in \mathcal{F}: \|f\|_{\mathcal{F}} \leq 1} \left(\mathbb{E} [f(Z)] - \mathbb{E} [f(W)] \right) \\ 
    &=  \| \psi_{p} - \psi_q \|_{\mathcal{F}} \\ 
    &= \bigg(  \mathbb{E} [k(Z,Z')] + \mathbb{E}[k(W,W')] - 2 \mathbb{E} [k(Z,W)]  \bigg)^{1/2}
\end{align*}
where $\psi_p := \int_{\mathcal{X}} k(x, \cdot) p(dx)$ is the Bochner integral, i.e., the mean embedding of $p$ into $\mathcal{F}$ \citep{DBLP:conf/alt/SmolaGSS07}, and $Z'$ (resp. $W'$) is a random variable with distribution $p$ (resp. $q$) and is independent of $Z$ (resp. $W$). In sequel, we interchangeably refer to MMD by $\text{MMD}(p,q; \mathcal{F}), \text{MMD}(p,q;k)$, or $\text{MMD}(p,q)$ if the context is clear. 

% \begin{prop}[\citep{DBLP:conf/nips/FukumizuGSS07,DBLP:journals/jmlr/GrettonBRSS12}]
% If the Bochner integral $\psi_p$ induced by $k(\cdot, \cdot)$ is injective, then MMD is a metric in $\mathcal{P}(\mathbb{X})$.
% \label{prop:mmd_is_a_metric}
% \end{prop}
% \noindent \textbf{Remark}. \textit{Characteristic} kernels (e.g., Gaussian and Laplace kernels) induce an injective Bochner integral. Any \textit{universal} kernel is also characteristic.

% Besides the popular Gaussian kernels, we also consider here rather less popular ones defined by $k_{\alpha}(x,y) := -\|x - y \|^{\alpha}, \forall \alpha \in \mathbb{R}, \forall x,y \in \mathcal{X}$. We refer to $k_{\alpha}$ as an \textit{unrectified} kernel with order $\alpha$ as in the special case when $\alpha=1$, $k_{\alpha}$ is the unrectified triangular kernel \citep{fleuret2003scale}. The unrectified kernel with order $\alpha \in (0,2)$ also induces a proper metric as formally stated in the following proposition.

% \begin{prop}[\citep{szekely2003statistics}]
% Let $k_{\alpha}(x,y) := -\|x - y \|^{\alpha}, \forall \alpha \in \mathbb{R}, \forall x,y \in \mathcal{X}$. $\text{MMD}(\cdot, \cdot; k_{\alpha})$, which is also called a generalized energy distance in this case, is a metric on $\mathcal{P}_{\alpha}(\mathcal{X})$ for all $\alpha \in (0,2)$ but not a metric for $\alpha=2$. 
% \end{prop}

\noindent \textbf{Empirical approximation}. 
Given empirical samples $\{z_i\}_{i=1}^N \sim p$ and $\{w_i\}_{i=1}^M \sim q$, MMD admits a simple empirical estimate as  
\begin{align*}
    &\text{MMD}_{b}^2(\{z_i\}, \{w_i\}; k) = \\
    &\frac{1}{N^2} \sum_{i,j} k(z_i, z_j) + \frac{1}{M^2} \sum_{i,j} k(w_i, w_j) - \frac{2}{N M} \sum_{i,j} k(z_i, w_j). 
    % \label{eq:empirical_mmd}
\end{align*}
Though there is also a simple unbiased estimate of $\text{MMD}$, the \emph{biased estimate} $\text{MMD}_{b}$ has smaller variance in practice and thus is adopted in our work. 

% It follows from the Mercer's theorem \citep{MercerFunctionsOP} that any continuous, symmetric, non-negative definite kernel $k$ can be expressed in feature dot product, i.e., there exits a feature map $\phi(\cdot)$ such that $k(x,y)= \langle \phi(x), \phi(y) \rangle$ for all $x,y \in \mathbb{X}$. Thus, the squared MMD in Equation (\ref{eq:population_mmd}) can be written as the mean difference in the feature space, i.e.,  $\texttt{MMD}^2(\mathcal{F}, p, q) = \| \mathbb{E}[\phi(Z)] - \mathbb{E}[\phi(W)] \|^2$. Minimizing the MMD distance becomes matching the means of two distributions if the feature map $\phi$ is the identity function; in other cases such as the Gaussian RBF kernel $k(x,y) = \exp(-\frac{1}{h} (x - y)^2)$, it is equivalent to match the moments of all orders between two distributions as the Taylor expansion of the Gaussian kernel covers all orders of moments. 

\subsection{Problem setting}
Consider $d=1$. For any policy $\pi$, let $\mu^{\pi}=\text{law}(Z^{\pi})$ be the law (distribution) of the return r.v. $Z^{\pi}$ as defined in Eq. (\ref{eq:return_rv}). The distributional Bellman operator $\mathcal{T}^{\pi}$ \citep{DBLP:conf/icml/BellemareDM17} specifies the relation of different return distributions across state-action pairs along the Bellman dynamic; that is, for any $\mu \in \mathcal{P}(\mathcal{X})^{\mathcal{S} \times \mathcal{A}}$, and any $(s,a) \in \mathcal{S} \times \mathcal{A}$,
% The distributional Bellman equation in policy evaluation specifies the relation of the return random variables across state-action pairs: 
\begin{align*}
    &\mathcal{T}^{\pi} \mu(x,a) := \nonumber\\
    &\int_{\mathcal{S}} \int_{\mathcal{A}} \int_{\mathcal{X}} (f_{\gamma,r})_{\#} \mu(s',a') \mathcal{R}(dr|s,a) \pi(d a'|s') P(ds'|s,a),
    % \label{eq:distributional_bellman_operator}
\end{align*}
where $f_{\gamma,r}(z) := r + \gamma z, \forall z$ and $(f_{\gamma,r})_{\#} \mu(s',a')$ is the pushforward measure of $\mu(s',a')$ by $f_{\gamma,r}$. Note that $\mu^{\pi}$ is the fixed point of $\mathcal{T}^{\pi}$, i.e., $\mathcal{T}^{\pi} \mu^{\pi} = \mu^{\pi}$. We are interested in the  problem of learning $\mu^{\pi} \in \mathcal{P}(\mathcal{X})^{\mathcal{S} \times \mathcal{A}}$ via the distributional Bellman operator $\mathcal{T}^{\pi}$. 
\subsection{Algorithmic approach}
In practical setting, we must approximate the return distribution via a finite set of statistics as the space of Borel probability measures is infinite-dimensional. Let $Z_{\theta}(s,a) := \{Z_{\theta}(s,a)_i\}_{i=1}^N$ be a set of parameterized statistics of $\mu^{\pi}(s,a)$ where $\theta$ represents the parameters of the model, e.g., neural networks. Instead of restricting $Z_{\theta}(s,a)$ to predefined statistic functionals, we model unrestricted statistics, i.e., deterministic samples where each $Z_{\theta}(s,a)_i$ can be evolved into any form of statistics and we use the Dirac mixture $\hat{\mu}_{\theta}{(s,a)} = \frac{1}{N} \sum_{i=1}^N \delta_{Z_{\theta}(s,a)_i}$ to approximate $\mu^{\pi}(s,a)$. We refer to the deterministic samples $Z_{\theta}(s,a)$ as particles, and our goal is reduced into learning the particles $Z_{\theta}(s,a)$ to approximate $\mu^{\pi}(s,a)$. To this end, the particles $Z_{\theta}(s,a)$ is deterministically evolved to minimize the MMD distance between the approximate distribution and its distributional Bellman target.
Algorithm \ref{alg:mmd-drl} below presents the generic update in our approach, namely MMDRL. 

\begin{algorithm}
\DontPrintSemicolon
\KwRequire{Number of particles $N$, kernel $k$, discount factor $\gamma \in [0,1]$}
\KwInput{Sample transition $(s, a, r, s')$} 
%   \KwInitialization{$\theta$}

% $Q(x', a') = \frac{1}{N}\sum_{i=1}^N Z_{\theta}(x', a')_i, \forall a' \in \mathcal{A}$ \; 
\If{Policy evaluation}
{
$a^* \sim \pi(\cdot|s')$
} 
\ElseIf{Control setting}{
% \tcc{Greedy w.r.t. the value function estimate}
$a^* \leftarrow \argmax_{a' \in \mathcal{A}}  \frac{1}{N}\sum_{i=1}^N Z_{\theta}(s', a')_i$ \; 
}

% \tcc{Compute an empirical Bellman target measure}

$\hat{T} Z_i \leftarrow r + \gamma Z_{\theta^{-}}(s', a^*)_i, \forall 1 \leq i \leq N$

\KwOutput{ $\text{MMD}_{b}^2\left(\{Z_{\theta}(s,a)_i\}_{i=1}^N, \{\hat{T} Z_i\}_{i=1}^N; k \right)$
}
\caption{\textbf{Generic MMDRL update}}
\label{alg:mmd-drl}
\end{algorithm}

\noindent \textbf{Intuition}. The MMDRL reduces into the standard TD or Q-learning when $N=1$. For $N > 1$, the objective $\text{MMD}_{b}$ when used with SGD and Gaussian kernels $k(x,y)=\exp(-|x-y|^2/h)$ contributes in two ways: (i) The term  $\frac{1}{N^2} \sum_{i,j} k(Z_{\theta}(s,a)_i, Z_{\theta}(s,a)_j)$ serves as a repulsive force that pushes the particles $\{Z_{\theta}(s,a)_i\}$ away from each other, preventing them from collapsing into a single mode, with force proportional to $\frac{2}{h}e^{-(Z_{\theta}(s,a)_i - Z_{\theta}(s,a)_j)^2/h}|Z_{\theta}(s,a)_i - Z_{\theta}(s,a)_j|$; (ii) the term $-\frac{2}{N^2} \sum_{i,j} k(Z_{\theta}(s,a)_i,\hat{T}Z_j)$ acts as an attractive force which pulls the particles $\{Z_{\theta}(s,a)_i\}$ closer to their target particles $\{\hat{T} Z_i\}$. This can also be intuitively viewed as a two-sample counterpart to Stein point variational inference \citep{DBLP:conf/nips/LiuW16,DBLP:conf/icml/ChenMGBO18}. 

\noindent \textbf{Particle representation}.
We can easily extend MMDRL to DQN-like architecture to create a novel deep RL, namely MMDQN (Algorithm 3 in Appendix B). In this work, we explicitly represent the particles $\{Z_{\theta}(s,a)_i\}_{i}^N$ in MMDQN via fixed $N$ network outputs as in QR-DQN \citep{DBLP:conf/aaai/DabneyRBM18} for simplicity (details in Appendix B). We emphasize that modeling improvements from IQN \citep{DBLP:conf/icml/DabneyOSM18} and FQF \citep{yang2019fully} can be naturally applied to MMDQN: we can implicitly generate $\{Z_{\theta}(s,a)_i\}_{i}^N$ via applying a neural network function to $N$ samples of a base sampling distribution (e.g., normal or uniform distribution) as in IQN, or we can use the proposal network in FQF to learn the weights of each Dirac components in MMDQN instead of using equal weights $1/N$.

\subsection{Theoretical Analysis}

% Since we are not working with one but a set of probability measures in $\mathcal{P}(\mathcal{X})^{\mathcal{S} \times \mathcal{A}}$, we are interested in the supremum form of MMD defined below. 
Here we provide theoretical understanding of MMDRL. Before that, we define the notion of supremum MMD, a MMD counterpart to the supremum Wasserstein in \citep{DBLP:conf/icml/BellemareDM17}, to work on $\mathcal{P}(\mathcal{X})^{\mathcal{S} \times \mathcal{A}}$. 
\begin{defn}
Supremum MMD is a functional $\mathcal{P}(\mathcal{X})^{\mathcal{S} \times \mathcal{A}} \times \mathcal{P}(\mathcal{X})^{\mathcal{S} \times \mathcal{A}} \rightarrow \mathbb{R}$ defined by 
\begin{align*}
    \text{MMD}_{\infty}(\mu, \nu; k) := \sup_{(x,a) \in \mathcal{S} \times \mathcal{A}} \text{MMD}(\mu(x,a), \nu(x,a); k)
\end{align*}
for any $\mu, \nu \in \mathcal{P}(\mathcal{X})^{\mathcal{S} \times \mathcal{A}}$. 
\end{defn}

We are concerned with the following questions:
\begin{enumerate}
    \item \textbf{Metric property}: When does $\text{MMD}_{\infty}$ induce a metric on $\mathcal{P}(\mathcal{X})^{\mathcal{S} \times \mathcal{A}}$?
    \item \textbf{Contraction property}: When is $\mathcal{T}^{\pi}$ a contraction in $\text{MMD}_{\infty}$? 
    
    \item \textbf{Convergence property}: How fast do the particles returned by minimizing MMD approach the target distribution it approximates?
\end{enumerate}
The metric property in the first question ensures that $\text{MMD}_{\infty}$ is a meaningful test to distinguish two return distributions on $\mathcal{P}(\mathcal{X})^{\mathcal{S} \times \mathcal{A}}$.
The contraction property in the second question guarantees that following from Banach's fixed point theorem \citep{BanachSURLO}, $\mathcal{T}^{\pi}$ has a unique fixed point which is $\mu^{\pi}$. In addition, starting with an arbitrary point $\mu_0 \in \mathcal{P}(\mathcal{X})^{\mathcal{X} \times \mathcal{A}}$, $\mathcal{T}^{\pi} \circ \mathcal{T}^{\pi} \circ ... \circ \mathcal{T}^{\pi} \mu_{0}$ converges at an exponential rate to $\mu^{\pi}$ in $\text{MMD}_{\infty}$. We provide sufficient conditions to answer the first two questions and derive the convergence rate of the optimal particles in approximating a target distribution for the third question. A short answer is that the first two properties highly depend on the underlying kernel $k$, and the particles returned by minimizing MMD enjoy a rate $O(1/\sqrt{n})$ regardless of the dimension $d$ of the underlying space $\mathcal{X}$. 

\subsubsection{Metric property}
\begin{prop}
Let $\tilde{\mathcal{P}}(\mathcal{S}) \subseteq \mathcal{P}(\mathcal{S})$ be some (Borel) subset of the space of the Borel probability measures. If $\text{MMD}$ is a metric on $\tilde{\mathcal{P}}(\mathcal{X})$, then $\text{MMD}_{\infty}$ is also a metric on $\tilde{\mathcal{P}}(\mathcal{X})^{\mathcal{S} \times \mathcal{A}}$. 
% If the Bochner integral $\psi_p$ induced by $k(\cdot, \cdot)$ is injective, $\text{MMD}_{\infty}$ is a metric in $\mathcal{P}(\mathbb{X})^{\mathcal{X} \times \mathcal{A}}$.
\label{prop:metric_over_state_action_space}
\end{prop}

\begin{proof}
See Appendix A.1. 
\end{proof}

Theorem \ref{theorem:metric_property} below provides \emph{sufficient conditions} for $\text{MMD}$ to induce a metric on $\tilde{\mathcal{P}}(\mathcal{X})$.

% For universal kernels such as the Gaussian kernel $k(x,y) = \exp(-\frac{1}{h} (x - y)^2)$, minimizing the MMD is equivalent to minimizing the distance between moments of all orders of $p$ and $q$. 

% \subsection{Sample-based distributional RL}
% We build up on the DQN architecture \citep{mnih2015human} with minimal changes to the base architecture. Specifcally, we use the identical neural network architecture of DQN except that we change the output layer size to $\mathcal{A} \times N$, as opposed to $\mathcal{A}$ in DQN. 
% We approach the distribution estimation problem in distributional RL from a two-sample perspective. 

\begin{thm}
We have 
\begin{enumerate}
    \item 
    If the underlying kernel $k$ is \textbf{characteristic} (i.e., the induced Bochner integral $\psi_p$ is injective), e.g., Gaussian kernels, then MMD is a metric on $\mathcal{P}(\mathcal{X})$ \citep{DBLP:conf/nips/FukumizuGSS07,DBLP:journals/jmlr/GrettonBRSS12}. 

    \item 
    Define \textbf{unrectified} kernels $k_{\alpha}(x,y) := -\|x - y \|^{\alpha}, \forall \alpha \in \mathbb{R}, \forall x,y \in \mathcal{X}$. Then $\text{MMD}(\cdot, \cdot; k_{\alpha})$ is a metric on $\mathcal{P}_{\alpha}(\mathcal{X})$ for all $\alpha \in (0,2)$ but not a metric for $\alpha=2$ \citep{szekely2003statistics}. 
    
     \item MMD associated with the so-called \textbf{exp-prod} kernel $k(x,y) = \exp( \frac{xy}{\sigma^2})$ for any $\sigma > 0$ is a metric on $\mathcal{P}_*(\mathcal{X})$.  
    
    % \item \textbf{Moment matching kernel}. Let $\sigma > 0$. The kernel $k(x,y) = \phi(x)^T \phi(y)$ where $\phi(x) = [ a_0, a_1, ..., a_n, ... ]^T$ with $a_n := \frac{1}{\sqrt{n!}} (\frac{x}{\sigma})^n$ is called a \textit{moment matching} kernel (for a reason that should be clear in the proof). Then, the MMD associated with a moment matching kernel is a metric on $\mathcal{P}_*(\mathcal{X})$.  
\end{enumerate}
\label{theorem:metric_property}
\end{thm}

\begin{proof}
See Appendix A.2. 
\end{proof}

\subsubsection{Contraction property}
We analyze the contraction of $\mathcal{T}^{\pi}$ for several important classes of kernels. One such class is shift invariant and scale sensitive kernels. A kernel $k(\cdot, \cdot)$ is said to be \textit{shift invariant} if $k(x+c, y + c) = k(x,y), \forall x,y,c \in \mathcal{X}$; it is said to be 
 \textit{scale sensitive} with order 
$\alpha >0$ if $k(c x, c y) = |c|^{\alpha} k(x,y), \forall x,y \in \mathcal{X}$ and $c \in \mathbb{R}$. For example, the unrectified kernel $k_{\alpha}$ considered in Theorem \ref{theorem:metric_property} is both shift invariant and scale sensitive with order $\alpha$ while Gaussian kernels are only shift invariant. 

\begin{thm} We have 
\begin{enumerate}
    \item If the underlying kernel is $k = \sum_{i \in I} c_i k_i$ where each component kernel $k_i$ is both shift invariant and scale sensitive with order $\alpha_i > 0$, $c_i \geq 0$, and $I$ is a (possibly infinite) index set, then $\mathcal{T}^{\pi}$ is a ${\gamma}^{ \alpha_* /2 }$-contraction in $\text{MMD}_{\infty}$ where $\alpha_* := \min_{i \in I} \alpha_i$.  
    
    \item $\mathcal{T}^{\pi}$ is \textbf{not} a contraction in $\text{MMD}_{\infty}$ associated with either Gaussian kernels or exp-prod kernels $k(x,y) = \exp( \frac{xy}{\sigma^2})$. 
    
    % \item If the underlying kernel is a moment matching kernel, then $\mathcal{T}^{\pi}$ is a $\sqrt{\gamma}$-contraction in $\text{MMD}_{\infty}$. 
    
    % \item If the underlying kernel is a Gaussian kernel, $\mathcal{T}^{\pi}$ is not a a contraction in $\text{MMD}_{\infty}$. 
\end{enumerate}
\label{theorem:contraction_property}
\end{thm}

\begin{proof}
See Appendix A.3. 
\end{proof}

\noindent \textbf{Practical consideration}. 
Theorem \ref{theorem:contraction_property} provides a negative result for the commonly used Gaussian kernel. In practice, however, we found that Gaussian kernels can promote to match the moments between two distributions and  have better empirical performance as compared to the other kernels analyzed in this section. In fact, MMD associated with Gaussian kernels $k(x,y) = \exp( -(x-y)^2 /(2 \sigma^2))$ can be decomposed into 
\begin{align*}
    \text{MMD}^2(\mu, \nu; k) = \sum_{n=0}^{\infty} \frac{1}{\sigma^{2n} n!} \left( \tilde{m}_n(\mu) - \tilde{m}_n(\nu) \right)^2
\end{align*}
where $\tilde{m}_n(\mu) =  \mathbb{E}_{x \sim \mu} \left[ e^{-x^2 /(2 \sigma^2 )} x^n \right]$, and similarly for $\tilde{m}_n(\nu)$. This indicates that MMD associated with Gaussian kernels approximately performs moment matching (scaled with a factor $e^{-x^2 /(2 \sigma^2 )}$ for each moment term).

\subsubsection{Convergence rate of distribution approximation} 
We justify the goodness of the particles obtained via minimizing MMD in terms of approximating a target distribution. 

% We prove that the deterministic points in our algorithm can simulate a target distribution. In addition, we also provide a convergence rate of distribution approximation by the deterministic points obtained via minimization of MMD.  
\begin{thm}
Let $\mathcal{X} \subseteq \mathbb{R}^d$ and $P \in \mathcal{P}(\mathcal{X})$. For any $n \in \mathbb{N}$, let $\{x_i\}_{i=1}^n \subset \mathcal{X}$ be a set of $n$ deterministic points such that $\{x_i\}_{i=1}^n \in \arginf_{\{\tilde{x}_i\}_{i=1}^n} \text{MMD}(\frac{1}{n} \sum_{i=1}^n \delta_{\tilde{x}_i}, P; \mathcal{F})$. Then, $P_n := \frac{1}{n} \sum_{i=1}^n \delta_{x_i}$ converges to $P$ at a rate of $O(1/\sqrt{n})$ in the sense that for any function $h$ in the unit ball of $\mathcal{F}$, we have 
\begin{align*}
    \bigg|\int_{\mathcal{X}} h(x) dP_n(x) - \int_{\mathcal{X}} h(x) dP(x)\bigg| = O(1/\sqrt{n}).
\end{align*}
\label{theorem:convergence_rate}
\end{thm}

\noindent \textbf{Remark}. MMD enjoys a convergence rate of $O(n^{-1/2})$ \footnote{In fact, the convergence rate can be further improved to $O(1/n)$ using kernel herding technique \citep{DBLP:conf/uai/ChenWS10}.} regardless of the underlying dimension $d$ while $1$-Wasserstein distance has a convergence rate of $O(n^{-1/d})$ (if $d > 2$) \citep{fournier2015rate}, which is slower for large $d$.  

% The empirical distribution $P_n$ obtained via MMD minimization converges to the true distribution $P$ at a rate of $n^{-1/2}$. Note that this convergence rate does not depend on the dimension $d$ of the underlying space $\mathcal{X}$ while $1$-Wasserstein distance has a convergence rate of $O(n^{-1/d})$, which is slower for large $d$.

% obtained by the $n$ deterministic points $\{x_i\}_{i=1}^n$ as a minimizer of the MMD distance 

\begin{proof}
% The main idea is that the $n$ deterministic samples construct an empirical distribution which is closer in the MMD distance to the true distribution than the empirical distribution from any $n$ i.i.d. samples from $P$. 

We first present two relevant results below (whose proofs are deferred to Appendix A.4) from which the theorem can follow. 

\begin{prop}
Let $(X_i)_{i=1}^n$ be $n$ i.i.d. samples of some distribution $P$. We have 
\begin{align*}
    \text{MMD}\left(\frac{1}{n} \sum_{i=1}^n \delta_{X_i}, P; \mathcal{F} \right)
    = O_p(1/\sqrt{n}),
\end{align*}
where $O_p$ denotes big-O in probability. 
\label{mmd_convergence_rate}
\end{prop}

\begin{lem}
Let $(a_n)_{n \in \mathbb{N}} \subset \mathbb{R}$ and $(X_n)_{n \in \mathbb{N}} \subset \mathbb{R}$ be sequences of deterministic variables and of random variables, respectively, such that for all $n$, $|a_n| \leq |X_n|$ almost surely (a.s.). Then, if $X_n = O_p(f(n))$ for some function $f(n) >0$, we have $a_n = O(f(n))$. 
\label{probability_bound_equals_standard_bound}
\end{lem}

It follows from the Cauchy-Schwartz inequality in $\mathcal{F}$ that for any function $h$ in the unit ball of $\mathcal{F}$, we have
\begin{align}
    &\bigg|\int_{\mathcal{X}} h(x) dP_n(x) - \int_{\mathcal{X}} h(x) dP(x)\bigg| \nonumber \\
    &\leq \| h\|_{\mathcal{F}} \times \bigg \|\int_{\mathcal{X}} k(x, \cdot) dP_n(x) - \int_{\mathcal{X}} k(x, \cdot) dP(x)\bigg \|_{\mathcal{F}} \nonumber \\
    &\leq \text{MMD}\left(\frac{1}{n} \sum_{i=1}^n \delta_{x_i}, P; \mathcal{F} \right). 
    \label{eq:moment_bounded_by_mmd}
\end{align}

Now by letting $X_n = \text{MMD}(\frac{1}{n} \sum_{i=1}^n \delta_{\tilde{x}_i}, P; \mathcal{F})$, $a_n = \text{MMD}(\frac{1}{n} \sum_{i=1}^n \delta_{x_i}, P; \mathcal{F})$, and $f(n) = 1/\sqrt{n}$, and noting that $a_n \leq X_n, \forall n$, Proposition \ref{mmd_convergence_rate}, Lemma \ref{probability_bound_equals_standard_bound} and Eq. (\ref{eq:moment_bounded_by_mmd}) immediately imply Theorem \ref{theorem:convergence_rate}. 
\end{proof}

\begin{figure*}[htb]
\centering
\includegraphics[scale=0.6]{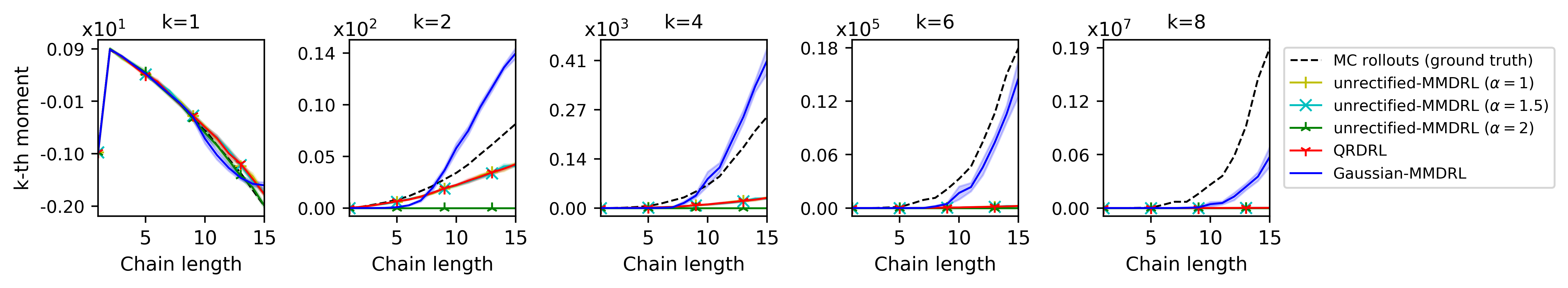}
\caption{Performance of different methods in approximating the optimal policy's return distribution at the initial state in the chain environment of various chain lengths $K = \{1,2,...,15\}$. The distribution approximation is evaluated in terms of how well a method can approximate the $k$-th central moment (except $k=1$ means the expectation) of the target distribution. $95\%$ C.I. with $30$ seeds. A variant (Gaussian-MMDRL) of our proposed MMDRL matches the MC rollouts (representing ground truth) much better than QRDRL.}
\label{fig:chain_experiment_result}
\end{figure*}

% \begin{figure}[h!]
%     \centering
%     \includegraphics[scale=0.5]{LaTeX/figs/chain.pdf}
%     \caption{An illustration of a variant of the classic chain MDP with the chain length $K$.}
%     \label{fig:chain_mdp}
% \end{figure}

\section{Experimental Results}

We first present results with a tabular version of MMDRL to illustrate its behaviour in distribution approximation task. We then combine the MMDRL update to the DQN-style architecture to create a novel deep RL algorithm namely MMDQN, and evaluate it on the Atari-57 games. We give full details of the architectures and hyperparameters used in the experiments in the Appendix B. \footnote{The code for the experiments is available at \url{https://github.com/thanhnguyentang/mmdrl}.}

\subsection{Tabular policy evaluation}

We empirically evaluate that MMDRL with Gaussian kernels (\textit{Gaussian-MMDRL}) can approximately learn the moments of a policy's return distribution as compared to the MMDRL with unrectified kernels (\textit{unrectified-MMDRL}) and the baseline QRDRL. 

We use a variant of the classic chain environment \citep{DBLP:conf/icml/RowlandDKMBD19} . The chain environment of length $K$ is a chain of $K$ states $s_0, ..., s_{K-1}$ where $s_0$ is the initial state and $s_{K-1}$ is the terminal state (see Figure \ref{fig:chain_mdp}). In each state, there are only two possible actions: (i) forward, which moves the agent one step to the right with probability $0.9$ and to $s_0$ with probability $0.1$, or (ii) backward, which transitions the agent to $s_0$ with probability $0.9$ and one step to the right with probability $0.1$. The agent receives reward $-1$ when transitioning to the initial state $s_0$, reward $1$ when reaching the terminal state $s_{K-1}$, and $0$ otherwise. The discount factor is $\gamma = 0.9$. We estimate $\mu^*_0$ the return distribution at the initial state of the optimal policy $\pi^*$ which selects forward action in every state. The longer the chain length $K$, the more stochastic the optimal policy's return distribution at $s_0$. We use $10,000$ Monte Carlo rollouts under policy $\pi^*$ to compute the central moments of  $\mu^*_0$ as ground truth values. Each method uses only $N=30$ samples to approximate the target distribution $\mu^*_0$ (more algorithm details are presented in Algorithm 2 in Appendix B).

% While QR assigns the $N$-sample budget to learn the quantile values of $\mu^*_0$ at fixed quantile location $\frac{2i-1}{2N}, 1 \leq i \leq N$, MMD learns $N$ unrestricted statistics which are more flexible. 

The result is presented in Figure \ref{fig:chain_experiment_result}. While all the methods approximate the expectation of the target distribution well, their approximation qualities differentiate greatly when it comes to higher order moments. Gaussian-MMDRL, though with only $N=30$ particles, can approximate higher order moments more reasonably in this example whereas the rest highly suffer from underestimation.  We also experimented with the kernel considered in Theorem \ref{theorem:metric_property}.3 in this tabular experiment and the Atari game experiment (next part) but found that it is highly inferior to the other kernel choices (even though it has an exact moment matching form as compared to Gaussian kernels, see Appendix A.2) thus we did not include it (we speculate that the shift invariance of Gaussian kernels seems effective when interacting with transition samples from the Bellman dynamics). 

\begin{figure}[h!]
    \centering
    \includegraphics[scale=0.55]{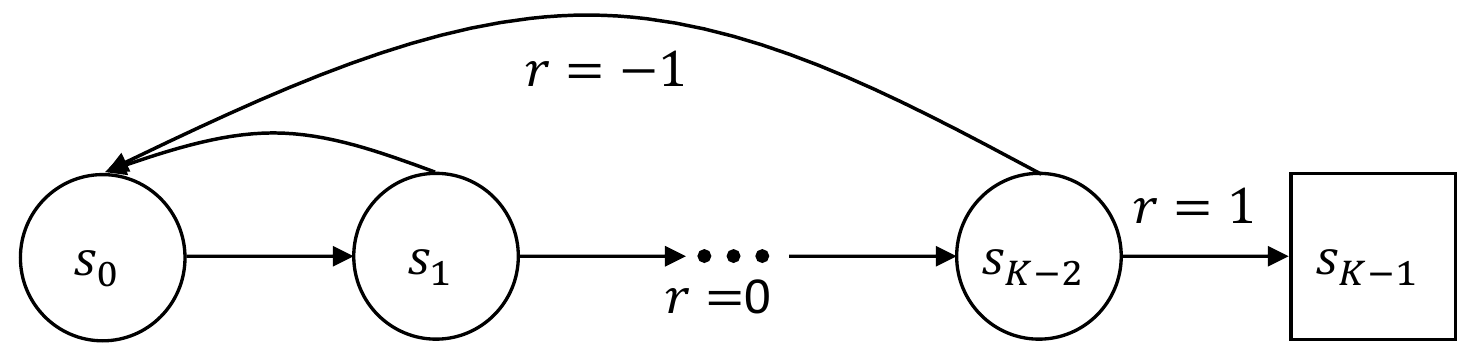}
    \caption{An illustration of a variant of the classic chain MDP with the chain length $K$.}
    \label{fig:chain_mdp}
\end{figure}

\subsection{Atari games}
\begin{figure*}[h]
    \centering
    \begin{minipage}[t]{0.49\textwidth}
        \centering
        \includegraphics[scale=0.56]{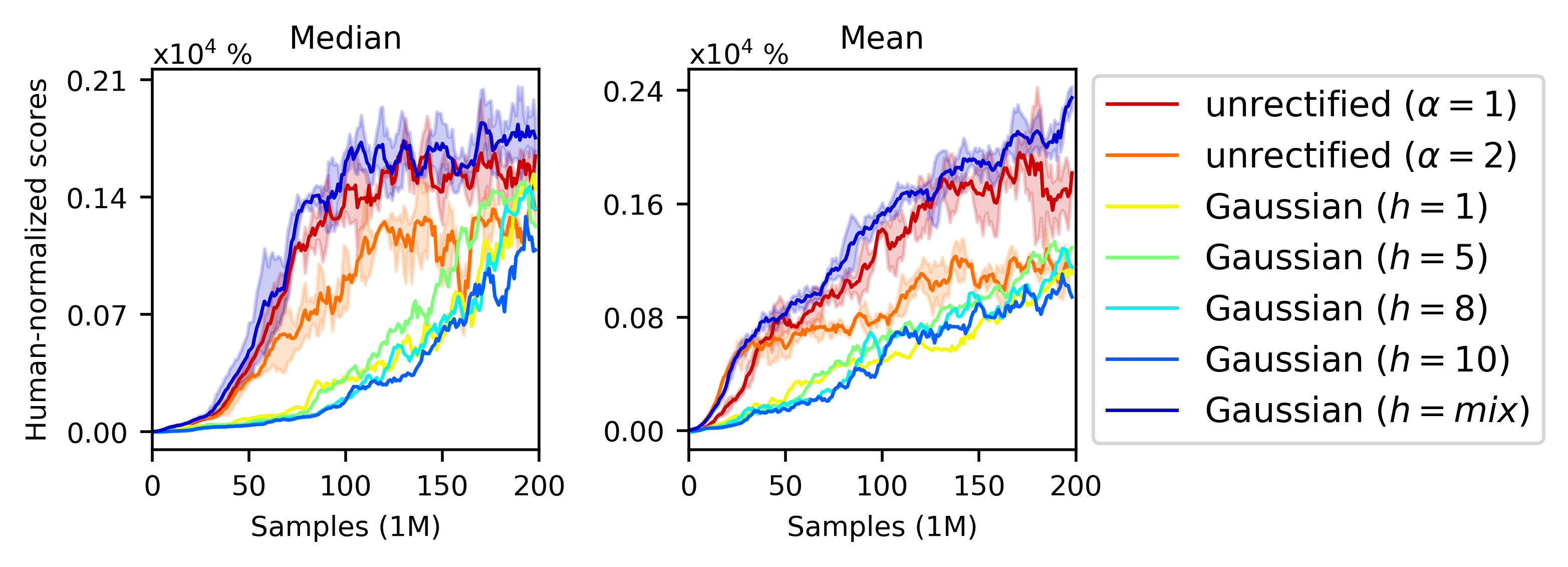} % first figure itself
        \subcaption{Different kernel bandwidths (at fixed $N=200$).}
    \end{minipage}\hfill
    \begin{minipage}[t]{0.49\textwidth}
        \centering
        \includegraphics[scale=0.56]{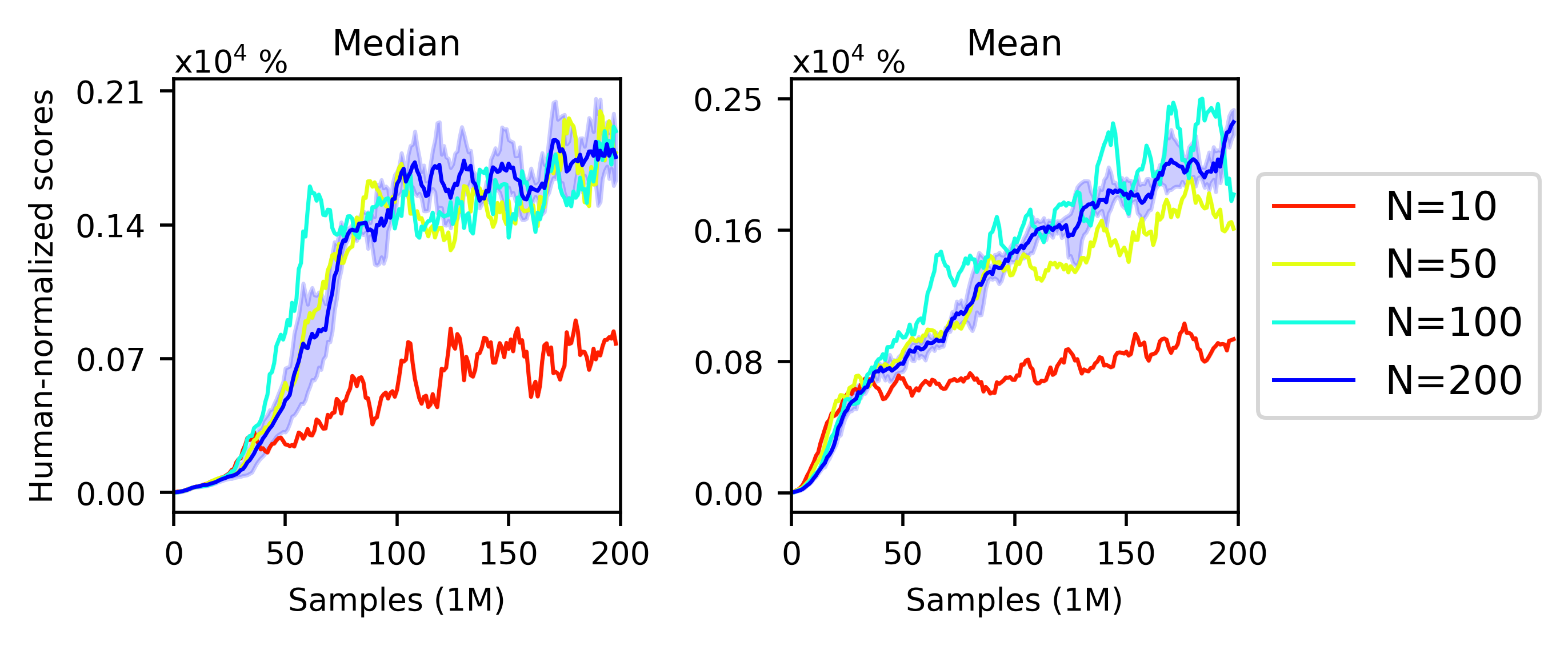} % second figure itself
        \subcaption{Different values of $N$ (at fixed $h=mix$).}
    \end{minipage}
    \caption{The sensitivity of (the human-normalized scores of) MMDQN in the 6 tuning games with respect to: (a) the kernel choices (Gaussian kernels with different bandwidths $h$ and unrectified kernels), and (b) the number of particles $N$. Here $h=mix$ indicates the mixture of bandwidth values in $\{1,2,...,10\}$. All curves are smoothed over $5$ consecutive iterations. $95\%$ C.I. for the $h=mix$, $N=200$, and unrectified kernel curves ($3$ seeds) and $1$ seed for the other curves.}
    
    \label{fig:hyperparameter_tunning}
\end{figure*}

To demonstrate the effectiveness of MMDRL at scale, we combine the MMDRL in Algorithm 1 with DQN-like architecture to obtain a deep RL agent namely MMDQN (Algorithm 3 in Appendix B). Specifically in this work, we used the same architecture as QR-DQN \citep{DBLP:conf/aaai/DabneyRBM18} for simplicity but more advanced modeling improvements from IQN \citep{DBLP:conf/icml/DabneyOSM18} and FQF \citep{yang2019fully} can naturally be used in combination to our framework. 

\noindent \textbf{Evaluation Protocol}. We evaluated our algorithm on 55 \footnote{We failed to include Defender and Surround games using OpenAI and Dopamine framework.} Atari 2600 games \citep{Bellemare_2013} following the standard training and evaluation procedures \citep{mnih2015human,DBLP:conf/aaai/HasseltGS16} (the full details are in appendix B). We computed human normalized scores for each agent per game. From the human normalized scores for an agent across all games, we extracted three statistics for the agent's performance: the median, the mean and the number of games where the agent's performance is above the human expert's performance. 

\noindent \textbf{Baselines}. 
We categorize the baselines into two groups. The first group contains \textit{comparable} methods: DQN, PRIOR., C51, and QR-DQN-1, where DQN \citep{mnih2015human} and PRIOR. (prioritized experience replay \citep{DBLP:journals/corr/SchaulQAS15}) are classic baselines. The second group includes \textit{reference} methods: RAINBOW \citep{DBLP:conf/aaai/HesselMHSODHPAS18}, IQN, and FQF, which contain algorithmic/modeling improvements orthogonal to MMDQN: RAINBOW combines C51 with prioritized replay and $n$-step update while IQN and FQF contain modeling improvements as described in the related work section. Since in this work we used the same architecture as QR-DQN and C51 for MMDQN, we directly compare MMDQN with the first group while including the second group for reference. 

% We directly compare our MMDQN to QR-DQN-1 \citep{DBLP:conf/aaai/DabneyRBM18} and C51 \citep{Bellemare_2013} as they all share the same network architecture. We also compare with prioritized experience replay \citep{DBLP:journals/corr/SchaulQAS15}. We include as reference methods IQN, Rainbow \citep{DBLP:conf/aaai/HesselMHSODHPAS18}, and FQF. These reference methods include several algorithmic/modeling advances: IQN uses implicit models to transform from the random samples of a base sampling distribution, Rainbow combines C51 with prioritized replay and $n$-step update, and FQF is similar to IQN but with optimized base samples instead of random ones. Since these improvements are orthogonal to our MMDQN, we consider IQN, Rainbow, FQF as reference methods rather than direct comparisons. 

% These reference methods include several algorithmic/modeling advances: IQN uses implicit models to transform from the random samples of a base sampling distribution, Rainbow combines C51 with prioritized replay and $n$-step update, and FQF is similar to IQN but with optimized base samples instead of random ones. 

% which uses implicit models to generate quantiles of the return distributions, Rainbow \citep{DBLP:conf/aaai/HesselMHSODHPAS18} which combines C51 with prioritized replay and n-step updates, and Fully parameterized Quantile Function (FQF) \citep{yang2019fully} which combines C51 with IQN. 

\noindent \textbf{Hyperparameter setting}. 
For fair comparison with QR-DQN, we used the same hyperparameters: $N=200$, Adam optimizer \citep{DBLP:journals/corr/KingmaB14} with $lr=0.00005, \epsilon_{ADAM} = 0.01/32$. We used $\epsilon$-greedy policy with $\epsilon$ being decayed at the same rate as in DQN but to a lower value $\epsilon=0.01$ as commonly used by the distributional RL methods. We used a target network to compute the distributional Bellman target as with DQN. Our implementation is based on OpenAI Gym \cite{brockman2016openai} and the Dopamine framework \citep{DBLP:journals/corr/abs-1812-06110}.

% and an optimal kernel selection is an active problem
\noindent \textbf{Kernel selection}. 
We used Gaussian kernels $k_h(x,y) = \exp \left( - (x-y)^2 /h \right)$ where $h>0$. The kernel bandwidth $h$ is crucial to the statistical quality of MMD: overestimated bandwidth results in a flat kernel while underestimated one makes the decision boundary highly irregular. 
% If the bandwidth is overestimated, the kernel is almost flat and the projection into the high-dimensional space becomes almost useless. If underestimated, the decision boundary becomes irregular and highly sensitive to noisy training data.  
We utilize the kernel mixture trick in \citep{DBLP:conf/icml/LiSZ15} which is a mixture of $K$ kernels covering a range of bandwidths $k(x,y) = \sum_{i=1}^K k_{h_i}(x,y)$. The Gaussian kernel with a bandwidth mixture yields much a better performance than that with individual bandwidth and unrectified kernels in 6 tuning games: Breakout, Assault, Asterix, MsPacman, Qbert, and BeamRider (see Figure \ref{fig:hyperparameter_tunning} (a)). Figure \ref{fig:hyperparameter_tunning} (b) shows the sensitivity of MMDQN in terms of the number of particles $N$ in the 6 tuning games where too small $N$ adversely affects the performance.

% In practice, the kernel bandwidth is heavily domain-dependent. For example, a common heuristic for the kernel bandwidth in hypothesis testing is the median trick  $h = 2 \texttt{med}^2$ where $\texttt{med}$ is the median distance of the aggregated points  \citep{DBLP:journals/jmlr/GrettonBRSS12}. \citeauthor{DBLP:conf/nips/SriperumbudurFGLS09} maximize the test statistic MMD over a family of kernels to effectively distinguish two distributions. 
% In \cite{DBLP:conf/nips/GrettonSSSBPF12}, the goal is to optimize the performance of the hypothesis test, thus they minimize the Type II error given an upper bound on Type I error. 
% In our setting, we instead use a mixture of $K$ kernels covering a range of bandwidths $k(x,y) = \sum_{i=1}^K k_{h_i}(x,y)$ . In practice, we informally searched on several bandwidth ranges and observed that the simple bandwidth range $\{1,2,..., 10\}$  yields good results. 

% Figure \ref{fig:hyperparameter_tunning} shows the sensitivity of MMDQN's performance with the number of particles $N$ and kernel choices in 6 tuning games: Breakout, Assault, Asterix, MsPacman, Qbert, and BeamRider.

The main empirical result is provided in Table \ref{tab:main_table} where we compute the mean and median of \textit{best} human normalized scores across 55 Atari games in the 30 no-op evaluation setting (the full raw scores for each game are provided in Appendix C). The table shows that MMDQN significantly outperforms the comparable methods in the first group (DQN, PRIOR., C51 and QR-DQN-1) in all metrics though it shares the same network architecture with C51 and QR-DQN-1. Although we did not include any orthogonal algorithmic/modelling improvements from the reference methods in the second group to MMDQN, MMDQN still performs comparably with these methods and even achieve a state-of-the-art mean human-normalized score.  In Figure \ref{fig:percentage_improvement} we also provide the percentage improvement per-game of MMDQN over QR-DQN-1 where MMDQN offers significant gains over QR-DQN-1 in a large array of games. 

% For reference methods in the second group (RAINBOW, IQN, FQF) which contains several algorithmic/modelling advances orthogonal to MMDQN, MMDQN does not only still performs comparably, and but also achieves a state-of-the-art mean human-normalized score.   

% The table shows that MMDQN significantly outperforms the standard distributional RL baselines QR-DQN-1 and C51 while enjoying several advantages such as simple modeling, no projection required, scalability and flexibility in statistical efficiency via the number of particles $N$. We also establish a new record on the mean human normalized score metric for non-distributed agents.  

% MMDQN outperforms the standard distributional RL baselines while enjoying several important advantages: simple modelling, no projection required, computational effectiveness, and flexibility in statistical efficiency via the number of particles $N$. In particular, MMDQN uses the same architecture of QR-DQN without any extra computational cost. 

% If the particles for each state-action pair are initialized to be identical, MMDQN is effectively the standard DQN. Fortunately this is not the case in practice due to randomized initialization in neural networks. We found that simple randomized initialization already gives diverse particles. 

% Tuning $N \in \{10, 50, 100, 200, 300\}$. Tuning in 7 games: Breakout, SpaceInvaders, Seaquest, ChopperCommand, StarGunner, MsPacMan, and CrazyClimber

\begin{table}[h]
    \centering
    \setlength{\extrarowheight}{1.5pt}
    \begin{tabular}{l|l|l|l|l}
         & \textbf{Mean} & \textbf{Median} & $>$\textbf{Human} & $>$\textbf{DQN}  \\
    \hline 
    DQN & 221\% & 79\% & 24 & 0 \\ 
    PRIOR. & 580\% & 124\% & 39 & 48\\ 
    C51 & 701\% & 178\% & 40 & 50 \\ 
    QR-DQN-1 & 902\% & 193\% & 41 & 54 \\ 
    \hline 
    RAINBOW & 1213\% & 227\% & 42 & 52 \\ 
    IQN & 1112\% & 218\% & 39 & 54 \\
    FQF & 1426\% & 272\% & 44 & 54  \\
    \hline 
    \textbf{MMDQN} & \textbf{1969}\% & 213\% & 41 & \textbf{55}
    \end{tabular}
    \caption{Mean and median of \textit{best} human-normalized scores across 55 Atari 2600 games. The results for MMDQN are averaged over 3 seeds and the reference results are from \citep{yang2019fully}.}
    % We compare MMDQN mainly against the first group. We also include the second group (RAINBOW, IQN, FQF) which use orthogonal improvements, implicit quantile networks or combine different distributional RL methods together
    \label{tab:main_table}
\end{table}

\begin{figure}[h!]
    \centering
    \includegraphics[scale=0.58]{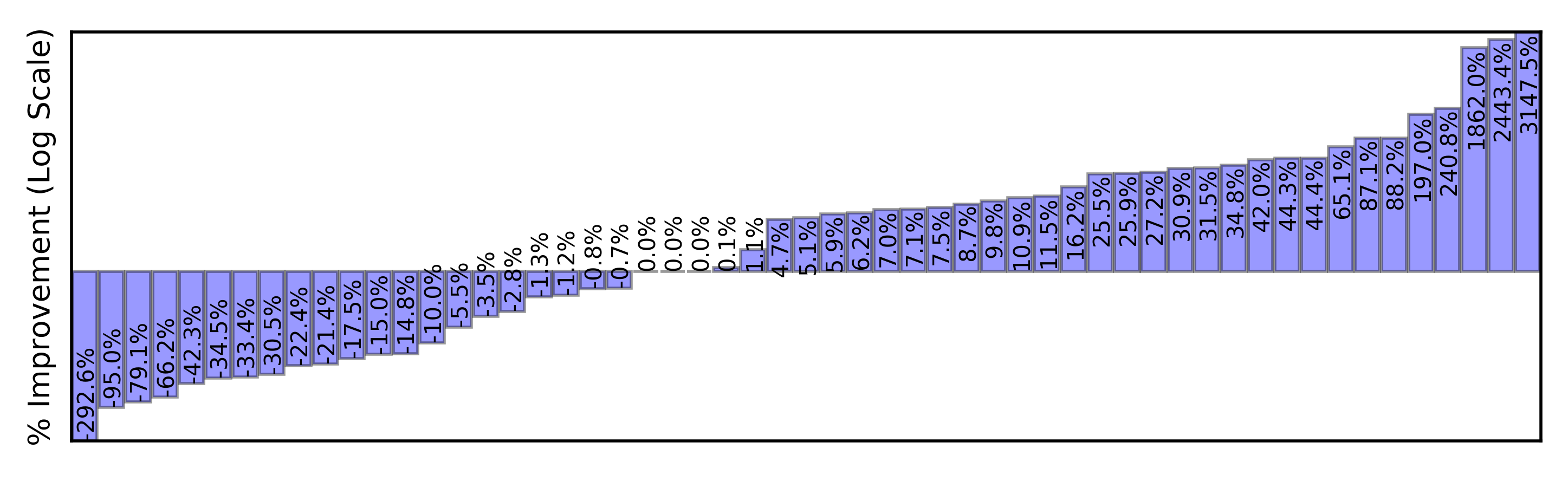}
    \caption{Percentage improvement per-game of MMDQN over QR-DQN-1.}
    \label{fig:percentage_improvement}
\end{figure}

\section{Conclusion and Discussion}
We have introduced a novel approach for distributional RL that eschews the predefined statistic principle used in the prior distributional RL. Our method deterministically evolves the (pseudo-)samples of a return distribution to approximately match moments of the resulting approximate distribution with those of the return distribution. We have also provided theoretical understanding of distributional RL within this framework. Our experimental results show that MMDQN, a combination of our approach with DQN-like architecture, achieves significant improvement in the Atari benchmark.

In what follow, we further discuss our results and the future work. \\ 
\noindent \textbf{Deterministic sampling in Theorem \ref{theorem:convergence_rate}}. Theorem \ref{theorem:convergence_rate} is concerned with the convergence of the optimal deterministic particles uniquely arisen in our problem setting where we deterministically evolve a set of particles to approximate a distribution in MMD. In other words, Theorem \ref{theorem:convergence_rate} does not fully analyze Algorithm \ref{alg:mmd-drl} but addresses one relevant yet important aspect: if a set of deterministic particles are evolved to simulate a distribution in MMD, how good is the approximation. This is different from the conventional setting in MMD which are often concerned with the convergence of empirical MMD derived from i.i.d. samples of each component distribution \citep{DBLP:journals/jmlr/GrettonBRSS12} (though we leverage similar proof techniques). We also remark that in Theorem \ref{theorem:convergence_rate}, we assume the attainability of the infimum but do not specify a practical algorithm to solve this infimum. While in practice, we use neural networks to represent the deterministic particles and use SGD to solve this optimization problem (as in MMDQN), the related literature of kernel herding can in fact provide a different approach with an improved analysis. Herding \citep{DBLP:conf/icml/Welling09} is a method that generates pseudo-samples (i.e., deterministic samples) from a distribution such that nonlinear moments of the sample set closely match those of the target distribution. A greedy selection of pseudo-samples can achieve a convergence rate of $O(1/n)$ \citep{DBLP:conf/uai/ChenWS10}. Different from greedy herding, MMDQN collectively find the set of pseudo-samples using SGD at each learning step. This collective herding by SGD is more effective in the distributional RL context than greedy herding as in distributional RL we need to perform herding for multiple distributions which themselves also evolve over learning steps. \\

\noindent \textbf{Predefined statistics vs. unrestricted statistics}. We further clarify that a specified functional form in Section ``Predefined statistic principle" section means that the parametric form $\zeta$, which possibly depends on some real parameters, is fully specified, e.g., quantile values $F^{-1}_{\eta}(\tau)$ (which depends on quantile level $\tau$) are predefined statistics. In this sense, FQF/IQN still have predefined statistics as they use quantile values to approximate a distribution. We emphasize that the flexibility of $\tau$ in FQF/IQN does not solve the problem of predefined statistics but rather only gives a finer-grained approximation by using more $\tau$ which in turn increases the budget of $N$ statistics. We remark that a flexible increase of $N$ (which can give a finer approximation in both predefined and unrestricted statistics) is not our focus in this paper; thus, we find it useful to keep the same fixed budget of $N$ statistics when comparing predefined and unrestricted statistics. We also remark that a general empirical sample can be considered an unrestricted statistic as it does not subscribe to any predefined parametric functional form. As a concrete example, let us consider the task of approximating a distribution with only one (learnable) statistic: One approach targets the median, and the other uses an unrestricted statistic. Eventually, the former approach should converge its statistic to the median as it subscribes to this predefined functional form while the latter approach can evolve its statistic into any sample equally (not necessarily the median) as long as the sample can simulate the distribution in a certain sense (in our case, to match the moments of the empirical distribution with those of the target distribution). \\

\noindent \textbf{Automatic kernel selection for MMDQN}. The kernel used in MMDQN plays a crucial role in achieving a good empirical performance and selecting the same kernel to perform well in all the games is a highly non-trivial task. Our current work uses a relatively simple but effective heuristics which uses a mixture of Gaussian kernels with different bandwidths. We speculate that a systematic way of selecting a kernel can even boost the empirical performance of MMDQN further. A promising direction is that instead of relying on a predefined kernel, we can train an adversarial kernel \citep{DBLP:conf/nips/SriperumbudurFGLS09,DBLP:conf/nips/LiCCYP17} to provide a stronger signal about a discrepancy between two underlying distributions; that is, $\min_{\theta \in \Theta} \max_{k \in \mathcal{K}} \text{MMD}( Z_{\theta}(x,a), [\mathcal{T}  Z_{\theta}](x,a); k), \forall (x,a)$ where $\mathcal{K}$ is a set of kernels. \\

\noindent \textbf{Modeling improvements for MMDQN}.  We focus our current work only on the statistical aspect of distributional RL and deliberately keep all the other design choices similar to the basic QR-DQN (e.g, we did not employ any modeling improvements and uncertainty-based exploration).  As our framework does not require the likelihood but only (pseudo-)samples from the
return distribution, it is natural to build an implicit generative model (as in IQN) for the return
distribution in MMDQN where we transform via a deterministic parametric function the samples
from a base distribution, e.g., a simple Gaussian distribution, to the samples of the return distribution. The weights of the empirical distribution can be made learnable by a proposal network as in FQF. \\ 

\noindent \textbf{An open question about the sufficient condition for contraction}. In this work, we prove that the distributional Bellman operator is not a contraction in MMD with Gaussian kernels using the scale-insensitivity of Gaussian kernels. On the other hands, we show that the distributional Bellman operator is a contraction in MMD with shift-invariant and scale-sensitive kernels.  This suggests a question of whether the scale sensitivity is a necessary condition for the contraction under MMD. Another direction is an understanding of a precise notion and the role of approximate contraction in practical setting as here the distributional Bellman operator is not a contraction in MMD with Gaussian kernels but the Gaussian kernels still give a favorable empirical performance in the Atari games as compared to the other kernels. \\ 

\noindent \textbf{Robust off-policy estimation in distributional RL}. Another potential direction from the current work is to estimate the return distributions merely from offline data generated by some behaviour policies. Since the estimation is constructed from finite offline data, robustness is a key to avoid a spurious estimation \citep{pmlr-v108-nguyen20a}. 
\section{Acknowledgement}
This research was partially funded by the Australian Government through the Australian Research Council (ARC). Prof. Venkatesh is the recipient of an ARC Australian Laureate Fellowship (FL170100006). We would also like to thank our anonymous reviewers (from NeurIPS'20 and AAAI'21) for the constructive feedback, Will Dabney (DeepMind) for the raw result data of QR-DQN-1 and the useful comments for our initial draft, and Mengyan Zhang (Australian National University) for the interesting discussion about order statistics. 
\bibliography{main}

\newpage
\onecolumn
\section{Appendix A. Proofs}

\subsection{A.1. Proof of Proposition 1}
\begin{proof}
If MMD is a metric in $\tilde{\mathcal{P}}(\mathcal{X})$. Then, it is obvious to see that $\text{MMD}_{\infty}(\mu,\nu) \geq 0, \forall \mu, \nu$ and that $\text{MMD}_{\infty}(\mu,\nu) = 0$ implies $\mu = \nu$. We now prove that $\texttt{MMD}_{\infty}$ satisfies the triangle inequality. Indeed, for any $\mu,\nu,\eta \in \tilde{\mathcal{P}}(\mathcal{X})^{\mathcal{S} \times \mathcal{A}}$, we have 
\begin{align*}
    \text{MMD}_{\infty}(\mu, \nu) &=  \sup_{(s,a) \in \mathcal{S} \times \mathcal{A}} \text{MMD}(\mu(s,a), \nu(s,a)) \\
    &\overset{(a)}{\leq} \sup_{(s,a) \in \mathcal{S} \times \mathcal{A}} \bigg\{ \text{MMD}(\mu(s,a), \eta(s,a)) + \text{MMD}(\eta(s,a), \nu(s,a)) \bigg\} \\ 
    &\overset{(b)}{\leq}  \sup_{(s,a) \in \mathcal{S} \times \mathcal{A}} \text{MMD}(\mu(s,a), \eta(s,a)) + \sup_{(s,a) \in \mathcal{S} \times \mathcal{A}} \text{MMD}(\eta(s,a), \nu(s,a)) \\
    &=  \text{MMD}_{\infty}(\mu, \eta) +  \text{MMD}_{\infty}(\eta, \nu),
\end{align*}
where $(a)$ follows from the triangle inequality for MMD and $(b)$ follows from that $ \sup (A + B) \leq \sup A + \sup B$ for any two sets $A$ and $B$ where $A+B := \{a + b: a \in A, b \in B\}$. 

\end{proof}

\subsection{A.2. Proof of Theorem 1}
We first present a relevant result for the proof. 
\begin{lem}
Let $\phi$ be the feature vector of $k$, i.e., $k(x,y) = \phi(x)^T \phi(y)$. Then, for any $\mu, \nu \in \mathcal{P}(\mathcal{X})$, we have 
\begin{align*}
    \text{MMD}(\mu; \nu; k) = \| u-v \|_2
\end{align*}
where $u = \mathbb{E}_{x \sim \mu} \phi(x)$ and $v = \mathbb{E}_{y \sim \nu} \phi(y)$.
\label{lemma:mmd_as_feature_mean_difference}
\end{lem}
\begin{proof}
Let $X, X' \overset{\text{i.i.d.}}{\sim} \mu$, $Y, Y' \overset{\text{i.i.d.}}{\sim} \nu$, and $X,X',Y,Y'$ are mutually independent. We have 
\begin{align*}
    \text{MMD}^2(\mu; \nu; k) &= \mathbb{E} \left[ k(X,X') \right] + \mathbb{E} \left[ k(Y,Y') \right] - 2 \left[ \mathbb{E} k(X,Y) \right] \\ 
    &= \mathbb{E} \left[ \phi(X)^T \phi(X') \right] + \mathbb{E} \left[ \phi(Y)^T \phi(Y') \right] - 2 \mathbb{E} \left[ \phi(X)^T \phi(Y) \right] \\ 
    &= u^T u + v^T v - 2 u^T v \\ 
    &= \|u - v \|^2. 
\end{align*}
\end{proof}

\begin{proof}[Proof of Theorem 1]
We prove only Theorem 1.3, as the proofs of Theorem 1.1 and 1.2 can be preferred to in  \citep{DBLP:conf/nips/FukumizuGSS07,DBLP:journals/jmlr/GrettonBRSS12} and \citep[c.f. Proposition 2]{szekely2003statistics}, respectively. For some $\sigma >0$, let $k(x,y) = \exp(xy/\sigma^2)$. Let $\phi(x) = [a_0(x), a_1(x), ..., a_n(x), ...]$ where $a_n(x) = \frac{1}{\sqrt{n!}} \frac{x^n}{\sigma^n}$. The Taylor expansion of $k$ yields $k(x,y) = \phi(x)^T \phi(y)$. It follows from Lemma \ref{lemma:mmd_as_feature_mean_difference} that 
\begin{align*}
    \text{MMD}^2(\mu, \nu; k) &= \| \mathbb{E} \phi(X) - \mathbb{E} \phi(Y) \|^2 \\ 
    &= \sum_{n=0}^{\infty} \frac{1}{\sigma^{2n} n!} \left( \mathbb{E}[X^n] - \mathbb{E}[Y^n] \right)^2, 
\end{align*}
for any $\mu, \nu \in \mathcal{P}(\mathcal{X})$ where $X \sim \mu$ and $Y \sim \nu$. It is easy to see that $\text{MMD}(\mu, \nu; k) \geq 0$ and it satisfies the triangle inequality. We only need to prove that for any $\mu, \nu \in \mathcal{P}_*(\mathcal{X})$, if $\text{MMD}(\mu, \nu; k) = 0$, then $\mu = \nu$. Indeed, assume $\text{MMD}(\mu, \nu; k) = 0$, then $\mu$ and $\nu$ have equal moments of all orders. Note that a distribution $\mu$ is uniquely determined by its characteristic function $g_{\mu}(t) = \mathbb{E} \left[ \exp(itX) \right], \forall t$. Let $m_n(\mu) = \mathbb{E}[X^n]$ be the $n$-th moment of $\mu$. Taylor expansion of $g_{\mu}$ yields 
\begin{align*}
    g_{\mu}(t) = \sum_{n=0}^{\infty} \frac{ i^n t^n m_n(\mu) }{n!},
\end{align*}
a power series which is valid only within its radius of convergence. The radius of convergence of this power series is 
\begin{align*}
    r = \frac{1}{\limsup_{n \rightarrow \infty} \bigg|\frac{ m_n(\mu) }{ n!} \bigg|^{1/n}}. 
\end{align*}
Since $\mu \in \mathcal{P}_*(\mathcal{X})$, we have 
\begin{align*}
    \limsup_{n \rightarrow \infty} \frac{ |m_n(\mu)|^{1/n} }{ n} = 0. 
\end{align*}
Using Stirling's formular, this indicates that $\limsup_{n \rightarrow \infty} \bigg|\frac{ m_n(\mu) }{ n!} \bigg|^{1/n} = 0$, or $r = \infty$. Hence, the set of all moments of a distribution on $\mathcal{P}_*(\mathcal{X})$ uniquely determines the distribution. This concludes our proof. 
\end{proof}

\subsection{A.3. Proof of Theorem 2}

\subsubsection{Proof of Theorem 2.1}
\begin{lem}
Let $(\mu_i)_{i \in I}$ and $(\nu_i)_{i \in I}$ be two sets of Borel probability measures in $\mathcal{X}$ over some indices $I$. Let $p$ be any distribution induced over $I$, then we have
\begin{align*}
    \text{MMD}^2 \left( \sum_i p_i \mu_i, \sum_i p_i \nu_i \right) \leq \sum_{i} p_i \text{MMD}^2(\mu_i, \nu_i)
    % \left \| \psi_{\sum_i p_i \mu_i} -  \psi_{\sum_i p_i \nu_i} \right \|^2_{\mathcal{H}} \leq \sum_{i} p_i  \left \| \psi_{\mu_i} - \psi_{\nu_i} \right \|^2_{\mathcal{H}}. 
    % % \label{eq:embedding_of_sum}
\end{align*}
\label{lemma:mixture_of_measures_scale_down_mmd}
\end{lem}
\begin{proof}
Denoting  $g_i = \psi_{\mu_i} - \psi_{\nu_i}, \forall i$, we have
\begin{align*}
    \text{MMD}^2 \left( \sum_i p_i \mu_i, \sum_i p_i \nu_i \right) &= \left\| \psi_{\sum_i p_i \mu_i} -  \psi_{\sum_i p_i \nu_i} \right \|^2_{\mathcal{F}} \\
    &\overset{(a)}{=} \left \|  \sum_i p_i (\psi_{\mu_i} - \psi_{\nu_i}) \right \|^2_{\mathcal{F}} \\
    &= \sum_i \langle p_i g_i, p_i g_i \rangle_{\mathcal{F}} + 2 \sum_{i \neq j} \langle p_i g_i, p_j g_j \rangle_{\mathcal{F}} \\
     &= \sum_i p_i^2 \langle g_i, g_i \rangle_{\mathcal{F}} + 2 \sum_{i \neq j} p_i p_j \langle  g_i, g_j \rangle_{\mathcal{F}} \\
     &\overset{(b)}{\leq} \sum_i p_i^2 \|g_i\|^2_{\mathcal{F}} + 2 \sum_{i \neq j} p_i p_j \|g_i \|_{\mathcal{F}} \|g_j\|_{\mathcal{F}} \\ 
    %  &= \left \|  \sum_i p_i g_i \right \|^2_{\mathcal{H}} \\ 
     &=  \left( \sum_i \sqrt{p_i} \sqrt{p_i} \|g_i\|_{\mathcal{F}} \right)^2 \\
     &\overset{(c)}{\leq} (\sum_i p_i) \sum_{i} p_i  \left \|g_i \right \|^2_{\mathcal{F}} \\
     &= \sum_{i} p_i  \left \| \psi_{\mu_i} - \psi_{\nu_i} \right \|^2_{\mathcal{F}} = \sum_{i} p_i \text{MMD}^2(\mu_i, \nu_i),
\end{align*}
where (a) follows from that the Bochner integral is linear (w.r.t. the probability measure argument), i.e., $\psi_{\sum_i p_i \mu_i} = \sum_i p_i \psi_{\mu_i}$, and both inequalities (b) and (c) follow from Cauchy-Schwartz inequality. 
\end{proof}

\begin{lem}
If $k$ is shift invariant and  scale sensitive with order $\alpha >0$. Then for any $\mu, \nu \in \mathcal{P}(\mathcal{X})$ and any $r, \gamma \in \mathbb{R}$, we have 
\begin{align*}
    \text{MMD}^2 \bigg((f_{r,\gamma})_{\#} \mu, (f_{r,\gamma})_{\#} \nu; k \bigg) = |\gamma|^{\alpha} \text{MMD}^2(\mu,\nu; k),
\end{align*}
where $f_{r,\gamma}(z) := r + \gamma z$ and $\#$ denotes pushforward operator. 
\label{lemma:take_out_shift_invariant_and_scale_sensitive}
\end{lem}

\begin{proof}
It follows from the closed-form expression of MMD distance that we have 
\begin{align*}
    \text{MMD}^2 \bigg((f_{r,\gamma})_{\#} \mu, (f_{r,\gamma})_{\#} \nu; k\bigg) 
    &= \int \int k(z,z') (f_{r,\gamma})_{\#} \mu(dz) (f_{r,\gamma})_{\#} \mu(dz') + \int \int k(w,w') (f_{r,\gamma})_{\#} \nu(dw) (f_{r,\gamma})_{\#} \nu(dw') \\
    &-2 \int \int k(z,w) (f_{r,\gamma})_{\#} \mu(dz) (f_{r,\gamma})_{\#} \nu(dw) \\
    &= \int \int k(r+ \gamma z,r + \gamma z')  \mu(dz) \mu(dz') + \int \int k( w,r + \gamma w') \nu(dw) \nu(dw') \\
    &-2 \int \int k(r + \gamma z,r + \gamma w)  \mu(dz) \nu(dw) \\
% \end{align*}
% \begin{align*}
    &= |\gamma|^{\alpha}  \int \int k(z,z')  \mu(dz) \mu(dz') + |\gamma|^{\alpha} \int \int k(w,w') \nu(dw) \nu(dw') \\
    &-2 |\gamma|^{\alpha} \int \int k(z,w)  \mu(dz) \nu(dw) \\ 
    &= |\gamma|^{\alpha} \text{MMD}^2(\mu,\nu; k). 
\end{align*}
\end{proof}

\begin{lem}
For any $\mu, \nu \in \mathcal{P}(\mathcal{X})$, $(\beta_i)_{i \in I} \subset \mathbb{R}$ for some indices $I$, we have 
\begin{align*}
    \text{MMD}^2(\mu, \nu; \sum_{i \in I} \beta_i k_i) = \sum_{i \in I} \beta_i \text{MMD}^2(\mu, \nu; k_i)
\end{align*}
\label{lemma:take_out_mixture_of_kernels}
\end{lem}

\begin{proof}
We have 
\begin{align*}
    \text{MMD}^2(\mu, \nu; \sum_{i \in I} \beta_i k_i) 
    &= \int \int \sum_{i \in I} \beta_i k_i(z,z') \mu(dz) \mu(dz') + \int \int \sum_{i \in I} \beta_i k_i(w,w') \nu(dw) \nu(dw') -2 \int \int \sum_{i \in I} \beta_i k_i(z,w) \mu(dz) \nu(dw) \\
    &= \sum_{i \in I} \beta_i \bigg ( 
    \int \int   k_i(z,z') \mu(dz) \mu(dz') + \int \int  k_i(w,w') \nu(dw) \nu(dw') -2 \int \int  k_i(z,w) \mu(dz) \nu(dw) \bigg ) \\
    &=  \sum_{i \in I} \beta_i \text{MMD}^2(\mu, \nu; k_i).
\end{align*}
\end{proof}

\begin{lem}
If $k$ is shift invariant and  scale sensitive with order $\alpha >0$, then 
\begin{align*}
    \text{MMD}_{\infty}(\mathcal{T}^{\pi} \mu, \mathcal{T}^{\pi} \nu; k) \leq \gamma^{\alpha/2} \text{MMD}_{\infty}(\mu, \nu; k), 
\end{align*}
for any $\mu,\nu \in \mathcal{P}(\mathcal{X})$ and any (stationary) policy $\pi$.
\label{lemma:shift_invariant_and_scale_sensitive_kernels_are_contractive}
\end{lem}
\begin{proof}
We have 
\begin{align*}
    &\text{MMD}^2 \bigg(\mathcal{T}^{\pi} \mu(s,a), \mathcal{T}^{\pi} \nu(s,a); k \bigg)\\ 
    &=\text{MMD}^2 \bigg(
    \int (f_{r,\gamma})_{\#} \mu(s',a') \pi(da'|s') P(ds'|s,a) \mathcal{R}(dr|s,a), \int (f_{r,\gamma})_{\#} \nu(s',a') \pi(da'|s') P(ds'|s,a) \mathcal{R}(dr|s,a); k
    \bigg )\\
    &\overset{(a)}{\leq} \int \text{MMD}^2 \bigg( (f_{r,\gamma})_{\#} \mu(s',a'), (f_{r,\gamma})_{\#} \nu(s',a'); k \bigg) \pi(da'|s') P(ds'|s,a) \mathcal{R}(dr|s,a) \\ 
    &\overset{(b)}{=} \gamma^{\alpha} \int \text{MMD}^2 (  \mu(s',a'),  \nu(s',a') ) \pi(da'|s') P(ds'|s,a) \mathcal{R}(dr|s,a) \\ 
    &\leq \gamma^{\alpha} \text{MMD}^2_{\infty}(\mu, \nu). 
\end{align*}
Here $(a)$ follows from Lemma \ref{lemma:mixture_of_measures_scale_down_mmd}, and $(b)$ follows from Lemma \ref{lemma:mixture_of_measures_scale_down_mmd}. Finally, the last inequality concludes the lemma by the definition of supremum MMD. 
\end{proof}

\begin{proof}[Proof of Theorem 2.1]
We have 
\begin{align*}
     \text{MMD}^2 \bigg(\mathcal{T}^{\pi} \mu(s,a), \mathcal{T}^{\pi} \nu(s,a); \sum_{i \in I} c_i k_i \bigg) 
     &\overset{(a)}{=} \sum_{i \in I} c_i \text{MMD}^2 \bigg(\mathcal{T}^{\pi} \mu(s,a), \mathcal{T}^{\pi} \nu(s,a); k_i \bigg) \\
     &\overset{(b)}{\leq} \sum_{i \in I} c_i \gamma^{\alpha_i} \text{MMD}^2 \bigg(\mu(s,a), \nu(s,a); k_i \bigg) \\ 
     &\overset{(c)}{\leq} \gamma^{\alpha_*} \sum_{i \in I} c_i  \text{MMD}^2 \bigg(\mu(s,a), \nu(s,a); k_i \bigg) \\  
     &\overset{(d)}{=} \gamma^{\alpha_*} \text{MMD}^2 \bigg(\mu(s,a), \nu(s,a); \sum_{i \in I} c_i k_i \bigg) \\
     &\leq \gamma^{\alpha_*} \text{MMD}^2_{\infty} \bigg(\mu, \nu; \sum_{i \in I} c_i k_i \bigg).
\end{align*}
Here, $(a)$ and $(d)$ follow from Lemma \ref{lemma:take_out_mixture_of_kernels}, $(b)$ follows from Lemma \ref{lemma:shift_invariant_and_scale_sensitive_kernels_are_contractive}, and $(c)$ follows from that $\gamma \leq 1$ and $c_i \geq 0$. Finally, the last inequality above concludes the proof by the definition of supremum MMD.
\end{proof}

\subsubsection{Proof of Theorem 2.2}
\begin{proof}[Proof of Theorem 2.2]

\begin{figure}
    \centering
    \includegraphics[scale=0.5]{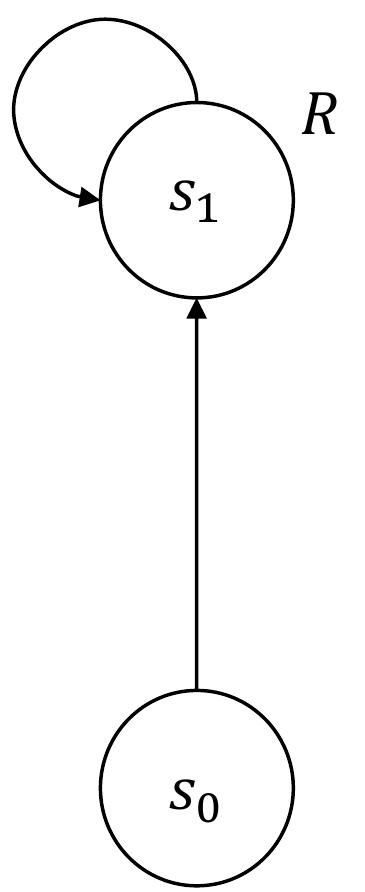}
    \caption{A simple MDP with 2 states: initial state $s_0$ and absorbing state $s_1$. Any agent receives a reward $r \sim R$ whenever it reaches state $s_1$.}
    \label{fig:simple_MDP}
\end{figure}

\textbf{Case 1}: Gaussian kernels $k(x,y) = \exp( -(x-y)^2 /(2 \sigma^2))$.  
\newline

We will prove that $\text{MMD}_{\infty}$ associated with Gaussian kernels $k(x,y) = \exp( -(x-y)^2 /(2 \sigma^2))$ for some  $\sigma >0$ is not a contraction by contradiction and counterexamples. Assume by contradiction that there exists some $\alpha > 0$ such that 
\begin{align}
    \text{MMD}_{\infty}(\mathcal{T}^{\pi} \mu, \mathcal{T}^{\pi} \nu; k) \leq \gamma^{\alpha} \text{MMD}_{\infty}(\mu, \nu;k), 
    \label{eq:contraction_contradiction_assumption}
\end{align}
for all $\mu, \nu \in \mathcal{P}(\mathcal{X})$. If $\alpha \geq 1$, $\gamma^{\alpha} \leq \gamma^{\alpha'}, \forall \alpha' \in (0,1]$, thus without loss of generality, assume that $\alpha \in (0,1]$. 

We provide a counterexample that contradicts Inequality (\ref{eq:contraction_contradiction_assumption}). Consider a simple MDP with only 2 states: initial state $s_0$ and absorbing state $s_1$ where an agent receives a reward $r \sim R$ whenever it reaches $s_1$ (see Figure \ref{fig:simple_MDP}). Assume that the reward distribution $R$ has $dom(R) = \{r_1, ..., r_n\}$ with respective probabilities $\{p_{r_1}=\epsilon, p_{r_2}=\epsilon, ..., p_{r_{n-1}}=\epsilon, p_{r_n} = 1 - (n-1) \epsilon\}$ for some $\epsilon \in (0, \frac{1}{n-1})$ to be chosen later. Let $\mu, \nu \in \mathcal{P}(dom(R))^{\mathcal{S}}$ where $\mathcal{S} = \{s_0, s_1\}$ such that $\mu(s_0) = \mu(s_1) = p := \sum_{i=1}^n p_i \delta_{r_i}$ (for $\{p_i\}$ to be chosen later) and $\nu(s_0) = \nu(s_1) = q: = \sum_{i=1}^n q_i \delta_{r_i}$ (for $\{q_i\}$ to be chosen later). It is easy to verify that for any policy $\pi$, we have 
\begin{align*}
    \mathcal{T}^{\pi} \mu(s_0) &= \mathcal{T}^{\pi} \mu(s_1) = \mathbb{E}_{r \sim R} \left[ (f_{r,\gamma})_{\#} p \right] \\
    \mathcal{T}^{\pi} \nu(s_0) &= \mathcal{T}^{\pi} \nu(s_1) = \mathbb{E}_{r \sim R} \left[ (f_{r,\gamma})_{\#} q  \right],
\end{align*}
where $f_{r,\gamma}(z) := r + \gamma z, \forall z$ and $\#$ denotes pushforward operation. Note that $(f_{r,\gamma})_{\#} \mu(s_1)$ assigns probabilities $\{p_1, ..., p_n\}$ respectively to $r + \gamma dom(R) := \{r + \gamma r_1, ..., r + \gamma r_n\}$, and similarly for  $(f_{r,\gamma})_{\#} \nu(s_1)$. We have 
\begin{align*}
    \sum_{i=1}^n p_{r_i}^2 = (n-1) \epsilon^2 + (1 - (n-1) \epsilon)^2 \in [\frac{1}{n}, 1). 
\end{align*}
Note that since $\alpha \in (0,1]$, we have $\gamma^{2 \alpha} \in [\gamma^2, 1)$. Now, choose $\gamma \in (0,1)$ such that $\gamma^2 \geq \frac{1}{n}$, then there exists $\epsilon \in (0, \frac{1}{n-1})$ such that $ \sum_{i=1}^n p_{r_i}^2 = \gamma^{2 \alpha}$. Define \begin{align*}
    \eta_r &:= (f_{r,\gamma})_{\#} p = \sum_{i=1}^n p_i \delta_{r + \gamma r_i}, \\ 
    \kappa_r &:= (f_{r,\gamma})_{\#} q = \sum_{i=1}^n q_i \delta_{r + \gamma r_i}.
\end{align*}
It follows from the closed form of MMD that 
\begin{align}
    \text{MMD}^2(\mathcal{T}^{\pi} \mu(s_0), \mathcal{T}^{\pi} \nu(s_0); k) 
    &= \text{MMD}^2(\mathbb{E}_r[\eta_r] , \mathbb{E}_r[\kappa_r] ; k) \nonumber \\
    &= \sum_{r,r' \in dom(R)} p_r p_{r'} \text{MMD}^2(\eta_r, \kappa_{r'}; k)  \nonumber \\ 
    \label{eq:true_up_to_here_for_contradiction_examples}
    &> \sum_{i=1}^n p_{r_i}^2  \text{MMD}^2(\eta_{r_i}, \kappa_{r_i}; k) \\
    &\overset{(a)}{=} (\sum_{i=1}^n p_{r_i}^2) \text{MMD}^2(\eta_0, \kappa_0; k) \nonumber \\
    &\overset{(b)}{=} \gamma^{2 \alpha} \text{MMD}^2(\eta_0, \kappa_0; k) \nonumber.
\end{align}
Here $(a)$ follows from that Gaussian kernels are shift invariant and $(b)$ follows from the particular choice of $\epsilon$ such that $ \sum_{i=1}^n p_{r_i}^2 = \gamma^{2 \alpha}$. Similarly, we have 
\begin{align*}
    \text{MMD}^2(\mathcal{T}^{\pi} \mu(s_1), \mathcal{T}^{\pi} \nu(s_1); k) > \gamma^{2 \alpha} \text{MMD}^2(\eta_0, \kappa_0; k).
\end{align*}
It remains to choose particular values of $\sigma, n, p, q, \{r_i\}, \gamma$ such that $\gamma^2 \geq \frac{1}{n}$ and $\text{MMD}^2(\eta_0, \kappa_0; k) \geq \text{MMD}^2(p, q; k)$. It is indeed possible by choosing the values as in Table \ref{tab:realization_of_parameters_for_counterexample}. 

\begin{table}[]
    \centering
    \begin{tabular}{l|l}
        \textbf{Parameters} & \textbf{Values} \\
        \hline 
        $n$ & $5$ \\
        $\gamma$ & $0.8$ (or $0.9$) \\ 
        $\sigma$ & $0.1$ \\ 
        $\{p_i\}$ & $[0.4, 0.3, 0.2, 0.1, 0]$ \\ 
        $\{q_i \}$ & $[0, 0.1, 0.2, 0.3, 0.4]$ \\ 
        $\{r_i\}$ & $[0, 0.25, 0.5, 0.75, 1.]$
    \end{tabular}
    \caption{A realization of the parameters in the counterexample for proving Theorem 2.2}
    \label{tab:realization_of_parameters_for_counterexample}
\end{table}

So far, we have constructed a particular instance such that 
\begin{align*}
    \text{MMD}(\mathcal{T}^{\pi} \mu(s_0), \mathcal{T}^{\pi} \nu(s_0); k) &> \gamma^{\alpha} \text{MMD}(\mu(s_0), \nu(s_0); k) \\ 
    \text{MMD}(\mathcal{T}^{\pi} \mu(s_1), \mathcal{T}^{\pi} \nu(s_1); k) &> \gamma^{\alpha} \text{MMD}(\mu(s_1), \nu(s_1); k). 
\end{align*}
Thus, we have 
\begin{align*}
    \text{MMD}_{\infty}(\mathcal{T}^{\pi} \mu, \mathcal{T}^{\pi} \nu; k) > \gamma^{\alpha} \text{MMD}_{\infty}(\mu, \nu;k), 
\end{align*}
which contradicts Inequality (\ref{eq:contraction_contradiction_assumption}). 

\noindent \textbf{Case 2}: For exp-prod kernels $k(x,y) = \exp( xy / \sigma^2)$. 

We follows the same procedure as in Case 1 but only up to Eq. (\ref{eq:true_up_to_here_for_contradiction_examples}) as the exp-prodkernel is not shift invariant. Instead, define 
\begin{align*}
    r^* = \argmin_{r \in \{r_1, ..., r_n\}} \text{MMD}(\eta_{r}, \kappa_{r}; k).
\end{align*}
Then, we have 
\begin{align*}
    \text{MMD}^2(\mathcal{T}^{\pi} \mu(s_0), \mathcal{T}^{\pi} \nu(s_0); k)
    &> \sum_{i=1}^n p_{r_i}^2  \text{MMD}^2(\eta_{r_i}, \kappa_{r_i}; k) \\
    &\geq (\sum_{i=1}^n p_{r_i}^2) \text{MMD}^2(\eta_{r^*}, \kappa_{r^*}; k) \\
    &= \gamma^{2 \alpha} \text{MMD}^2(\eta_{r^*}, \kappa_{r^*}; k).
\end{align*}

Now it remains to choose particular values of $\sigma, n, p, q, \{r_i\}, \gamma$ such that $\gamma^2 \geq \frac{1}{n}$ and $\text{MMD}^2(\eta_{r^*}, \kappa_{r^*}; k) \geq \text{MMD}^2(p, q; k)$. In fact, the values chosen for Case 1 as in Table \ref{tab:realization_of_parameters_for_counterexample} already yields the previous inequalities for Case 2.

\end{proof}

\subsection{A.4. Proofs of Lemma \ref{probability_bound_equals_standard_bound} and Proposition \ref{mmd_convergence_rate}}

\subsubsection{Proof of Lemma \ref{probability_bound_equals_standard_bound}}
\begin{proof}
For all $n$, we have 
\begin{align}
    \frac{|a_n|}{f(n)} \leq \frac{|X_n|}{f(n)} \text{ a.s. }
    \label{eq:an_leq_Xn}
\end{align}
Denote by $P$ the probability measure of the underlying measurable space defining the random variables $X_n$. Since $X_n = O_p(f(n))$, for any $\epsilon > 0$, there exists $M_{\epsilon} > 0, N_{\epsilon} > 0$ such that 
\begin{align*}
    P\left( \frac{|X_n|}{f(n)} > M_{\epsilon} \right) \leq \epsilon, \forall n > N_{\epsilon}.
\end{align*}
It follows from Eq. (\ref{eq:an_leq_Xn}) that $\forall n, \left\{ \frac{|a_n|}{f(n)} \leq M_{\epsilon} \right\} \supseteq \left\{ \frac{|X_n|}{f(n)} \leq M_{\epsilon} \right\}$ which implies that for all $n > N_{\epsilon}$, we have
\begin{align*}
    1_{\left\{ \frac{|a_n|}{f(n)} \leq M_{\epsilon} \right\}} &= P\left( \frac{|a_n|}{f(n)} \leq M_{\epsilon} \right) \\
    &\geq P\left( \frac{|X_n|}{f(n)} \leq M_{\epsilon} \right) \\
    &\geq 1 - \epsilon, 
\end{align*}
where the equality is due to that $a_n$ are deterministic. Picking any $\epsilon \in (0,1)$, we have 
\begin{align*}
    \frac{|a_n|}{f(n)} \leq M_{\epsilon}, \forall n > M_{\epsilon},
\end{align*}
which implies $a_n = O(f(n))$. 
\end{proof}

\subsubsection{Proof of Proposition \ref{mmd_convergence_rate}}
First, we state the following proposition. 
\begin{prop}
Assume that $\sup_{x,y} k(x,y) \leq B$. Let $X_1, ..., X_n$ be $n$ i.i.d. samples of $P$ and denote $\xi(X_1, ..., X_n) := \text{MMD}\left(\frac{1}{n} \sum_{i=1}^n \delta_{X_i}, P; \mathcal{H} \right)$. For any $\epsilon > 0$, we have 
\begin{align*}
    P\left(\xi(X_1, ..., X_n) > \epsilon + 2 \sqrt{ \frac{B}{n} } \right) \leq \exp \left(  -\frac{\epsilon^2 n}{2B} \right).
\end{align*}
\label{prop:mmd_final_bound}
\end{prop}
Note that by setting $t = \exp \left(  -\frac{\epsilon^2 n}{2B} \right)$ in Proposition \ref{prop:mmd_final_bound}, we have 
\begin{align*}
     P\left(\xi(X_1, ..., X_n) /(1/\sqrt{n}) > M_t \right) \leq t,
\end{align*}
where $M_t := \sqrt{2B \log(1/t)} + 2 \sqrt{B}$. Thus, $\xi(X_1,...,X_n) = O_p(1/\sqrt{n})$, which concludes Proposition \ref{mmd_convergence_rate}.

Now, we only need to prove Proposition \ref{prop:mmd_final_bound}. The proof follows a standard procedure to bound an empirical process $\xi(X_1, ..., X_n) $ where we first bound it in probability w.r.t. its expectation using concentration inequalities and then we bound its expectation via a complexity notation of the witness function class $\mathcal{H}$.  

\noindent \textbf{Preliminaries}. Before proving Proposition \ref{prop:mmd_final_bound}, we present some relevant notations and preliminary results from which we combine to derive a proof for Proposition \ref{prop:mmd_final_bound}. For any function $g: \mathcal{X}^n \rightarrow \mathbb{R}$, denote 
\begin{align*}
    \delta_i g(x_1, ...,x_n) &:= \sup_{z} g(x_1, ...,x_{i-1}, z, x_{i+1}, ...,x_n) - \inf_{z} g(x_1, ...,x_{i-1}, z, x_{i+1}, ...,x_n), \\
    \|\delta_i g\|_{\infty} &:= \sup_{x_1,...,x_n} |\delta_i g(x_1, ...,x_n)|. 
\end{align*}
Denote by $\mathcal{H} := \{f \in \mathcal{F}: \|f\|_{\mathcal{F}} \leq 1\}$ the unit ball of the RKHS $\mathcal{F}$. For any $\{x_i\}_{i=1}^n \in \mathcal{X}^n$, we denote $\mathcal{H} \circ \{x_1, ..., x_n\} := \{ (f(x_1), ..., f(x_n)) \in \mathbb{R}^n: f \in \mathcal{H} \}$. 

\begin{lem}[\textbf{McDiarmid's inequality} \citep{afol}]
Let $X_1, ..., X_n$ be independent random variables in $\mathcal{X}$. Let $g: \mathcal{X}^n \rightarrow \mathbb{R}$ be any function such that $\|\delta_i g\|_{\infty} < \infty, \forall i$. For any $\epsilon > 0$, we have 
\begin{align*}
    P\left( g(X_1, ...,X_n) - \mathbb{E} g(X_1,...,X_n) \geq \epsilon \right) \leq \exp\left(  -\frac{2 \epsilon^2}{ \sum_{i=1}^n \|\delta_i g\|_{\infty}^2 } \right). 
\end{align*}
\end{lem}

\begin{defn}
The Rademacher complexity of a set $\mathcal{T} \subseteq \mathbb{R}^n$ is defined as 
\begin{align*}
    Rad(\mathcal{T}): = \mathbb{E} \sup_{t \in \mathcal{T}} \frac{1}{n} \sum_{i=1}^n \Omega_i t_i,
\end{align*}
where $\Omega_1, ..., \Omega_n \in \{-1,1\}$ are independent Rademacher random variables, i.e., $P(\Omega_i = 1) = P(\Omega_i = -1) = 1/2, \forall i$.
\end{defn}

\begin{lem}
We have 
\begin{align*}
    \mathbb{E} \left[\xi(X_1, ..., X_n) \right] \leq 2 \mathbb{E} Rad( \mathcal{H} \circ \{X_1, ..., X_n\} ).
\end{align*}
In addition, we have 
\begin{align*}
    \mathbb{E} Rad( \mathcal{H} \circ \{X_1, ..., X_n\} ) \leq \sqrt{ \frac{B}{n} }.
\end{align*}
\label{lemma:bound_in_expectation}
\end{lem}
\begin{proof}
Let $\tilde{X}_1, ..., \tilde{X}_n$ be new independent samples from $P$ and independent of $S := \{X_1, ..., X_n\}$.
We have 
\begin{align*}
    \mathbb{E} \left[\xi(X_1, ..., X_n) \right] = \mathbb{E} \left[ \text{MMD}\left(\frac{1}{n} \sum_{i=1}^n \delta_{X_i}, P; \mathcal{F} \right) \right] 
    &= \mathbb{E} \sup_{f \in \mathcal{H}} \left( \mathbb{E} f(X) - \frac{1}{n} \sum_{i=1}^n f(X_i) \right) \\
    &= \mathbb{E} \sup_{f \in \mathcal{H}} \frac{1}{n} \sum_{i=1}^n \mathbb{E} \left[ f(\tilde{X}_i) - f(X_i) \bigg | S \right] \\ 
    &\overset{(a)}{\leq} \mathbb{E} \mathbb{E}  \sup_{f \in \mathcal{H}} \frac{1}{n} \sum_{i=1}^n  \left( f(\tilde{X}_i) - f(X_i) \right) \\ 
    &= \mathbb{E}  \sup_{f \in \mathcal{H}} \frac{1}{n} \sum_{i=1}^n \left( f(\tilde{X}_i) - f(X_i) \right) \\ 
    &\overset{(b)}{=} \mathbb{E}  \sup_{f \in \mathcal{H}} \frac{1}{n} \sum_{i=1}^n \Omega_i \left( f(\tilde{X}_i) - f(X_i) \right) \\ 
    &\leq \mathbb{E}  \left[ \sup_{f \in \mathcal{H}} \frac{1}{n} \sum_{i=1}^n \Omega_i f(\tilde{X}_i) + \sup_{f \in \mathcal{H}} \frac{1}{n} \sum_{i=1}^n (-\Omega_i) f(X_i) \right] \\
    &= 2 \mathbb{E}  \sup_{f \in \mathcal{H}} \frac{1}{n} \sum_{i=1}^n \Omega_i f(X_i) \\
    &= 2 \mathbb{E} Rad(\mathcal{H} \circ \{X_1, ..., X_n\}). 
\end{align*}
Here $(a)$ follows from Jensen's inequality for convex function $sup$ and $(b)$ follows from $\left(f(\tilde{X}_i) - f(X_i) \right)_{1 \leq i \leq n}$ has the same distribution as $\left(\Omega_i (f(\tilde{X}_i) - f(X_i)) \right)_{1 \leq i \leq n}$. 

In addition, we have 
\begin{align*}
    \mathbb{E} Rad(\mathcal{H} \circ \{X_1, ..., X_n\}) &= \mathbb{E}  \sup_{f \in \mathcal{H}} \frac{1}{n} \sum_{i=1}^n \Omega_i f(X_i) \\ &= \mathbb{E} \sup_{f \in \mathcal{H}} \bigg \langle f, \frac{1}{n} \sum_{i=1}^n \Omega_i k(X_i, \cdot) \bigg \rangle_{\mathcal{F}} \\ 
    &\leq \mathbb{E} \sup_{f \in \mathcal{H}} \|f\|_{\mathcal{F}} \times \bigg\|  \frac{1}{n} \sum_{i=1}^n \Omega_i k(X_i, \cdot) \bigg \|_{\mathcal{F}} \text{(Cauchy-Schwartz inequality)}\\
\end{align*}
\begin{align*}
    &\leq \mathbb{E} \bigg\|  \frac{1}{n} \sum_{i=1}^n \Omega_i k(X_i, \cdot) \bigg \|_{\mathcal{F}} \\ 
    &= \frac{1}{n} \mathbb{E} \sqrt{  \sum_{i,j} \Omega_i \Omega_j k(X_i, X_j)  } \\
    &\leq \frac{1}{n} \sqrt{ \mathbb{E} \sum_{i,j} \Omega_i \Omega_j k(X_i, X_j)  } \text{ (Jensen's inequality)}\\ 
    &= \frac{1}{n} \sqrt{ \mathbb{E} \sum_{i=1}^n k(X_i,X_i) } \leq \frac{1}{n} \sqrt{ nB} = \sqrt{\frac{B}{n}}.
\end{align*}
The last equality is due to that $\Omega_i^2 = 1, \forall i$ and $\mathbb{E} \left[ \Omega_i \Omega_j \right] = 0, \forall i \neq j$.
\end{proof}

\begin{proof}[Proof of Proposition \ref{prop:mmd_final_bound}]
Now, we are ready to prove Proposition \ref{prop:mmd_final_bound}. 
For any $\{x_i\}_{i=1}^n \in \mathcal{X}^n$, we have 
\begin{align*}
    \delta_i \xi(x_1, ..., x_n) 
    &\leq \sup_{z} \bigg \| \frac{1}{n} \sum_{j=1, j \neq i}^n k(x_j, \cdot) + \frac{k(z, \cdot)}{n} - \int k(x, \cdot) dP(x) \bigg \|_{\mathcal{F}} \\ 
    &- \inf_{z} \bigg \| \frac{1}{n} \sum_{j=1, j \neq i}^n k(x_j, \cdot) +  \frac{k(z, \cdot)}{n} - \int k(x, \cdot) dP(x) \bigg \|_{\mathcal{F}} \\ 
    &\leq \sup_{z,z'} \bigg \| \frac{k(z, \cdot)}{n} - \frac{k(z', \cdot)}{n} \bigg\|_{\mathcal{F}} \leq \frac{2\sqrt{B}}{n}. 
\end{align*}
Thus, it follows from McDiarmid's inequality that 
\begin{align}
    P( \xi(X_1, ...,X_n) - \mathbb{E} \xi(X_1,...,X_n) \geq \epsilon) \leq \exp \left( - \frac{2 n \epsilon^2}{B} \right). 
    \label{eq:bound_in_probability}
\end{align}
Eq. (\ref{eq:bound_in_probability}) and Lemma \ref{lemma:bound_in_expectation} immediately yield Proposition \ref{prop:mmd_final_bound}.
\end{proof}

\section{Appendix B. Algorithm details}
In this appendix, we present the algorithm details for the tabular and Atari experiments in the main text. 

\subsection{Tabular experiment}
In tabular case, the particles $Z_{\theta}(s,a)$ reduces to tabular representation $\{ \theta_i(s,a)\}_{i=1}^N$, so a return distribution is represented as a mixture of Diracs
\begin{align*}
    \mu(s,a) = \frac{1}{N} \sum_{i=1}^N \delta_{\theta_i(s,a)}. 
\end{align*}
The algorithm details used for tabular policy evaluation in the main text is presented in Algorithm \ref{alg:tabular_policy_evaluation}. In both tabular MMD-DRL and QRDRL cases, the TD gradients have a closed-form expression which we explicitly used in our tabular experiment. The detailed values of each (hyper-)parameters of the algorithms used in our experiment are reported in Table \ref{tab:tabular_policy_evaluation_params}. 
\begin{table}[h!]
    \centering
    \begin{tabular}{l|c} 
        \textbf{(Hyper-)Parameters} & \textbf{Values} \\
        \hline
        Learning rate schedule & $\alpha_t = \frac{1}{t^{0.2}}$  \\
        Particle initialization & $\mathcal{N}(-1, 0.08)$ \\ 
        Number of episodes per iteration & $100$ \\ 
        Number of iterations & $15$ \\ 
        Number of particles $N$ & $30$ \\ 
        Number of MC rollouts & $10,000$ \\ 
        Kernel bandwidth $h$ (MMDRL only) & $\{8,10,12\}$ \\ 
        Quantiles (QRDRL only) & $\{ \frac{2i-1}{2N}: 1 \leq i \leq N \}$
    \end{tabular}
    \caption{The (hyper-)parameters of the algorithms used in our tabular policy evaluation experiment with the Chain MDP.}
    \label{tab:tabular_policy_evaluation_params}
\end{table}

\begin{algorithm}
\DontPrintSemicolon
\KwRequire{Number of particles $N$, kernel $k$, discount factor $\gamma \in [0,1]$, evaluation policy $\pi$, learning rate $\alpha_t$, tabular particles $\{\theta_i(s,a)\}_{i=1}^N$}

\KwInitialization{initial particles $\theta$, initial copy particles $\theta^{-} \leftarrow \theta$, and initial state $s_0$}

\For{t=1,2,...}{
Take action $a_t = \pi(s_t)$ and observe $s_{t+1} \sim P(\cdot|s_t, a_t)$ and $r_t \sim \mathcal{R}(s_t, a_t)$ 

Compute Bellman target particles 
\begin{align*}
    \hat{T} \theta_i^{-} \leftarrow r_t + \gamma \theta^{-}_i(s_{t+1}, \pi(s_{t+1})), \forall i \in \{1,...,N\} 
\end{align*}

Compute TD gradient 
\begin{align*}
    &\text{\textbf{MMDRL Case}: } g_i \leftarrow \frac{\partial}{ \partial \theta_i(s_t,a_t)} \text{MMD}_b^2( \{\theta_i(s_t,a_t)\}_{i=1}^N,  \{\hat{T} \theta_i^{-}\}_{i=1}^N; k), \forall i \in \{1,...,N\} \\ 
    &\text{\textbf{QRDRL Case}: }   g_i \leftarrow \frac{\partial}{ \partial \theta_i(s_t,a_t)} \frac{1}{N} \sum_{j=1}^N \left( \hat{T} \theta_j^{-} - \theta_i(s_t,a_t) \right) \left( \frac{2i -1}{2N} - 1_{\{ \hat{T}\theta_j^{-} <  \theta_i(s_t,a_t)\}}  \right), \forall i \in \{1,...,N\}
\end{align*}

Update 
\begin{align*}
    \theta_i(s_t, a_t) &\leftarrow \theta_i(s_t, a_t) - \alpha_t g_i,\forall i \in \{1,...,N\}\\ 
    \theta^{-} &\leftarrow \theta 
\end{align*}

}
\KwOutput{ Approximate distribution $\mu(s,a) = \frac{1}{N} \sum_{i=1}^N \delta_{\theta_i(s,a)}$
}
\caption{\textbf{Tabular policy evaluation using TD-style update}}
\label{alg:tabular_policy_evaluation}
\end{algorithm}

\subsection{Atari game experiment}

We extend MMDRL to DQN-like architecture to create a novel deep RL, namely MMDQN. We use the same architecture of DQN except that we change the last layer to the size of $N \times |\mathcal{A}|$, instead of the size $|\mathcal{A}|$. In addition, we replace the squared loss in DQN by the empirical MMD loss. The full details for MMDQN are provided in Algorithm \ref{alg:mmd-dqn}. We expect that our framework would also benefit from recent orthogonal improvements to DQN such as double-DQN \citep{DBLP:conf/aaai/HasseltGS16}, the dueling architecture \citep{DBLP:conf/icml/WangSHHLF16} and prioritized replay \citep{DBLP:journals/corr/SchaulQAS15} but did not include these for simplicity. In Table \ref{tab:hyperparameters}, we provide the hyperparameter details of QR-DQN and MMDQN used in the Atari games. The hyperparameters in MMDQN that share with QR-DQN are intentionally set the same to allow for fair comparison.

\begin{table}[h]
    \centering
    \setlength{\extrarowheight}{4pt}
    \begin{adjustbox}{width=0.48\textwidth}
    \begin{tabular}{l|c|c}
    \textbf{Hyperparameters}     &  \textbf{QR-DQN} & \textbf{MMDQN} \\
    \hline 
    Learning rate       & $0.00005$ & $0.00005$ \\ 
    Optimizer & Adam & Adam \\ 
    $\epsilon_{ADAM}$ & $0.0003125$ & $0.0003125$ \\ 
    $N$ & $200$ & $200$ \\ 
    Quantiles & $\{ \frac{2i-1}{2N}: 1 \leq i \leq N \}$ & N/A \\
    Kernel bandwidth & N/A & $\{1,2,...,9,10\}$ \\
    \end{tabular}
    \end{adjustbox}
    \caption{The MMDQN hyperparameters as compared to those of QR-DQN.}
    \label{tab:hyperparameters}
\end{table}

\begin{algorithm}[h!]
\DontPrintSemicolon
\KwRequire{Number of particles $N$, kernel $k$ (e.g., Gaussian kernel), discount factor $\gamma \in [0,1]$, learning rate $\alpha$, replay buffer $\mathcal{M}$, main network $Z_{\theta}$, target network $Z_{\theta^{-}}$, and a policy $\pi$ (e.g., $\epsilon$-greedy policy w.r.t. $Q_{\theta}(s,a) =  \frac{1}{N}\sum_{i=1}^N Z_{\theta}(s, a)_i, \forall s,a$)}
% \KwInput{Sample transition $(x, a, r, x')$} 
%   \KwInitialization{$\theta$}

Initialize $\theta$ and $\theta^{-} \leftarrow \theta$  \; 

\For{t = 1,2,...}{

Take action $a_t \sim \pi(\cdot | s_t; \theta)$, receive reward $r_t \sim \mathcal{R}(\cdot| s_t, a_t)$, and observe $s_{t+1} \sim P(\cdot| s_t, a_t)$  \;

Store $(s_t, a_t, r_t, s_{t+1})$ to the replay buffer $\mathcal{M}$ \; 

Randomly draw a batch of transition samples $(s,a,r,s')$ from the replay buffer $\mathcal{M}$\; 

Compute a greedy action 
\begin{align*}
    a^* \leftarrow \argmax_{a' \in \mathcal{A}}  \frac{1}{N}\sum_{i=1}^N Z_{\theta^{-}}(s', a')_i
\end{align*} 

Compute the empirical Bellman target measure 
\begin{align*}
    \hat{T} Z_i^{-} \leftarrow r + \gamma Z_{\theta^{-}}(s', a^*)_i, \forall i \in \{1,...,N\}
\end{align*}

Update the main network 
\begin{align*}
    \theta \leftarrow \theta - \alpha \Delta_{\theta} \text{MMD}_b \left(\{Z_{\theta}(s,a)_i\}_{i=1}^N, \{\hat{T} Z_i^{-}\}_{i=1}^N; k \right)
\end{align*}
\; 

Periodically update the target network $\theta^{-} \leftarrow \theta$ 

}

\caption{\textbf{MMDQN}}
\label{alg:mmd-dqn}
\end{algorithm}

 We evaluated our algorithm on 55 Atari 2600 games \citep{Bellemare_2013} following the standard training and evaluation procedures \citep{mnih2015human,DBLP:conf/aaai/HasseltGS16}. For every 1M training steps in the environment, we computed the average scores of the agent by freezing the learning and evaluating the latest agent for 500K frames. We truncated episodes at 108K frames (equivalent to 30 minutes of game playing). We used the 30 no-op evaluation settings where we play a random number (up to 30) of no-op actions at the beginning of each episode during evaluation. We report the best score for a game by an algorithm which is the algorithm’s highest evaluation score in that game across all evaluation iterations during the training course (given the same hyperparameters are shared for all games). % (in both training and evaluation). 

The human normalized scores of an agent per game is the agent's normalized scores such that 0\% corresponds to a random agent and 100\% corresponds to the average score of a human expert. The human-normalized scores used in the paper are explicitly defined by 
\begin{align*}
    score = \frac{agent - random}{human - random},
\end{align*}
where $agent, human, random$ denotes the raw scores (undiscounted returns) for the given agent, the reference human player and the random player \cite{mnih2015human}, resp.,  in each game. 

The per-game percentage improvement of MMDQN over QR-DQN is computed as follow 
\begin{align*}
    percentage\_improvement = \frac{ score_{MMDQN} - score_{QR-DQN}  }{score_{QR-DQN}} \times 100 \%,
\end{align*}
where $ score_{MMDQN}$ and $score_{QR-DQN}$ are the best raw evaluation score of MMDQN and QR-DQN in the considered game. The log-scaled percentage improvement is computed as 
\begin{align*}
    log\_percentage\_improvement = 1_{\{ percentage\_improvement \geq 0 \}} \log(|percentage\_improvement| + 1).
\end{align*}

\section{Appendix C. Further experimental results} 
In Figure \ref{fig:vis_mmd_dqn_breakout} we visualize the behaviour of our MMDQN in the Breakout game. Three rows correspond to  3 consecutive frames of the Breakout game accompanied by the approximate return distributions learnt by MMDQN. Since the particles learnt by MMDQN represent empirical samples of the return distributions, we can visualize the return distributions via the learnt particles by plotting the histogram (with $17$ bins in this example) of these particles. The learnt particles in MMDQN can maintain diversity in approximating the return distributions even though there is no order statistics in MMDQN as in the existing distributional RL methods such as QR-DQN. The 3 consecutive frames illustrate that the ball is moving away from the left to the right. In response, MMDQN also moves the paddle away from the left by gradually placing the probability mass of the return for the LEFT action towards smaller values. In particular, in the first frame where the ball is still far away from the ground, the MMDQN agent does not make a significant difference between actions. As the ball is moving closer the ground from the left (the second and third frame), the agent becomes clearer that the LEFT action is not beneficial, thus placing the action's probability mass to smaller values.  

\begin{figure}
\centering
\subfloat{%
  \includegraphics[scale=0.8]{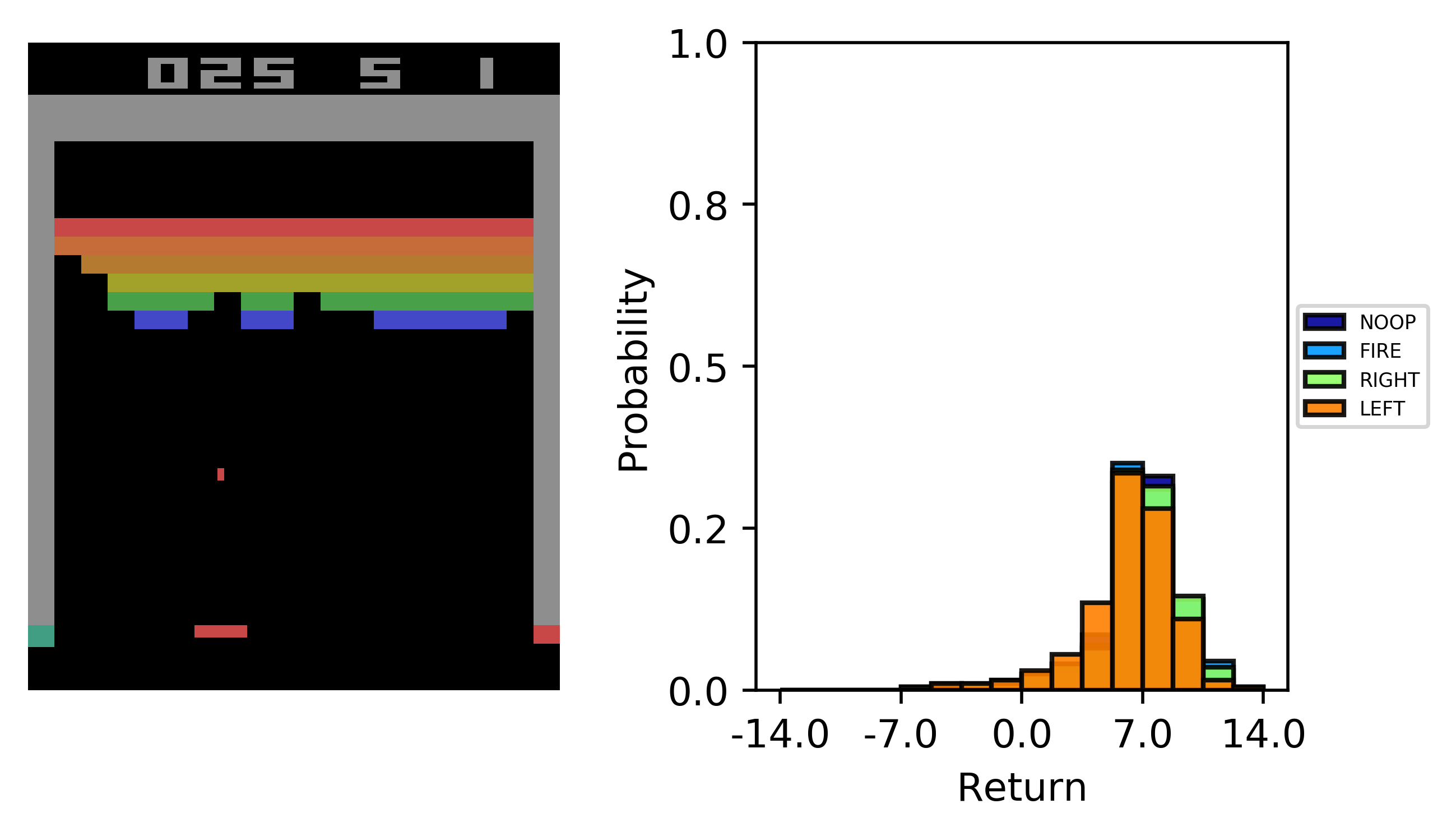}%
}

\subfloat{%
  \includegraphics[scale=0.8]{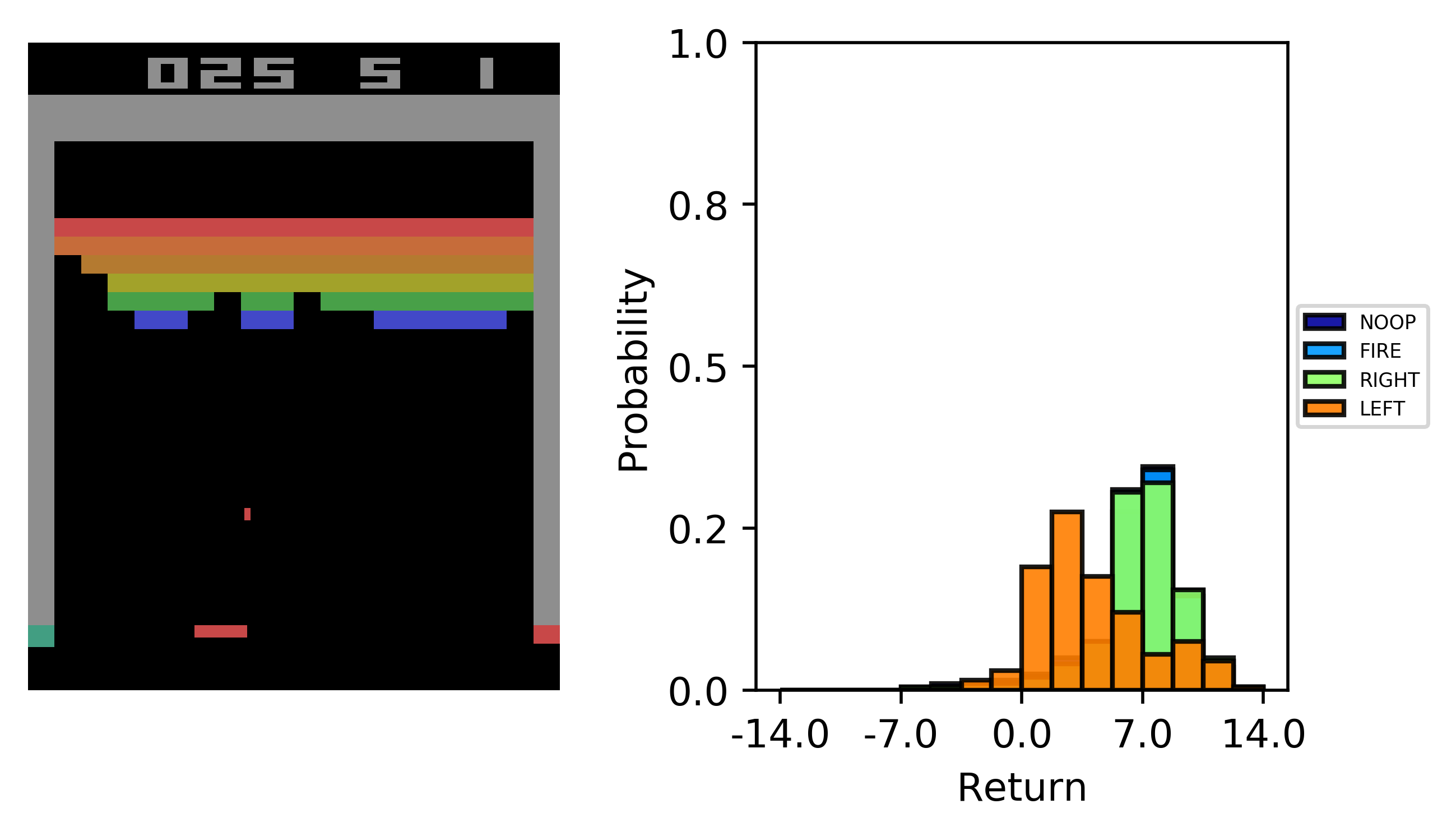}%
}

\subfloat{%
  \includegraphics[scale=0.8]{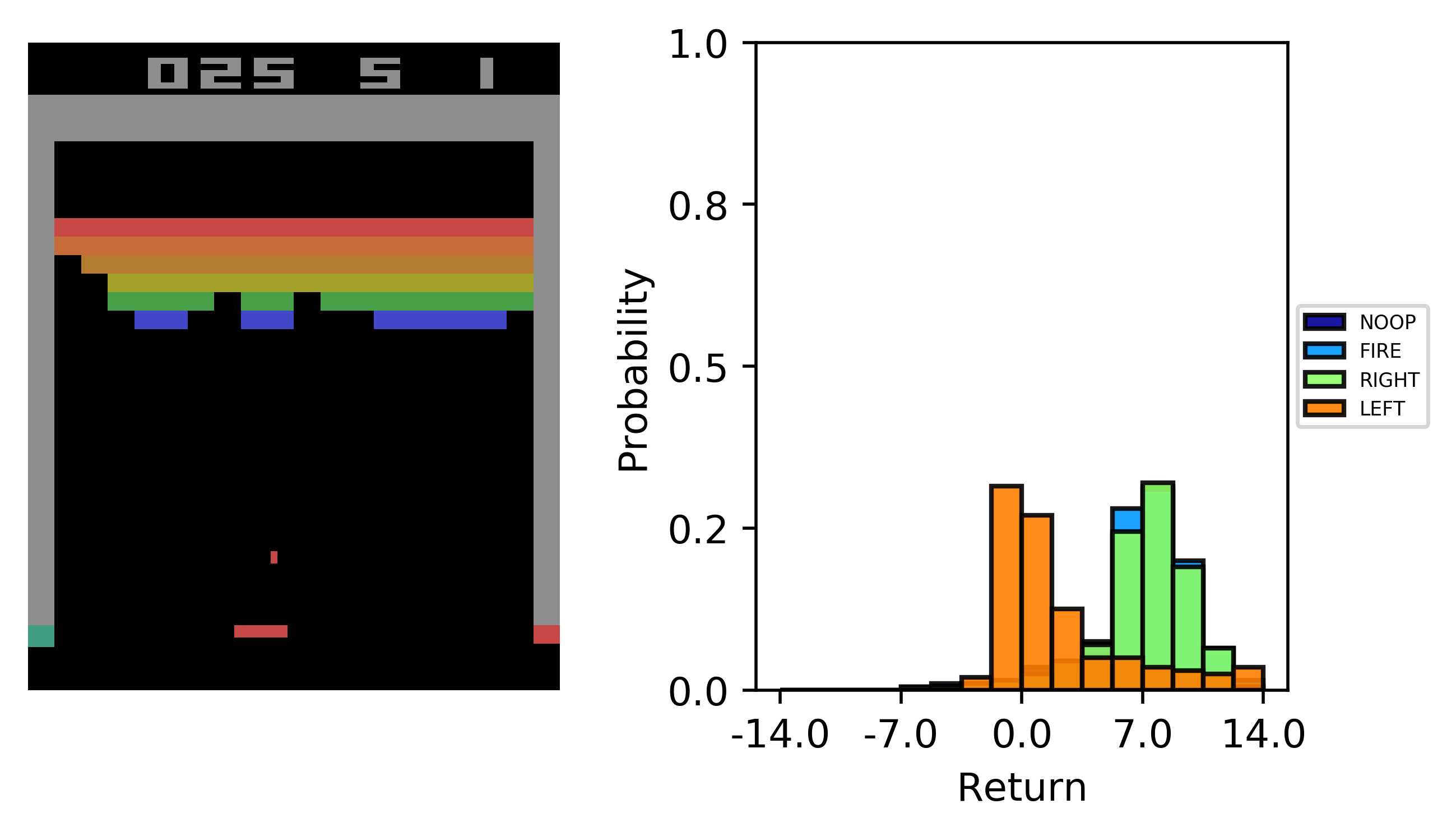}%
}

\caption{This example shows 3 consecutive frames and the approximate return distributions for all the actions in the Breakout game played by our MMDQN. The approximate return distributions plotted here are the histograms with $17$ bins constructed from the learnt particles by MMDQN.}
\label{fig:vis_mmd_dqn_breakout}
\end{figure}

\begin{figure}
    \centering
    \includegraphics[scale=0.8]{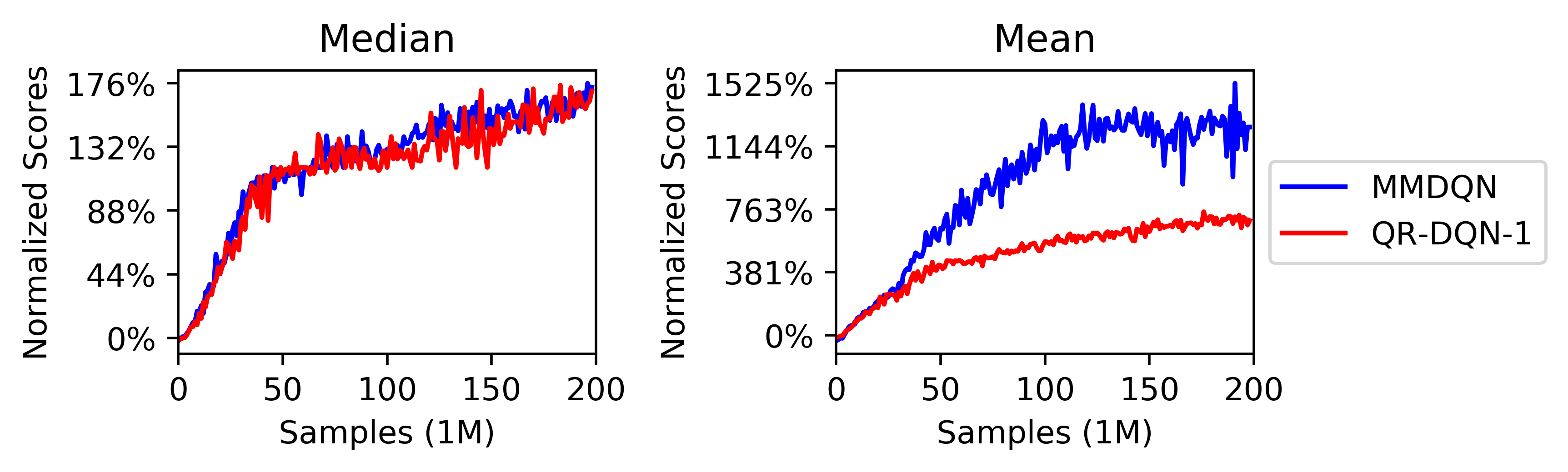}
    \caption{Median and mean of the test human-normalized scores across 55 Atari games for MMDQN (averaged over 3 seeds) and QR-DQN-1 (averaged over 2 seeds). }
    \label{fig:med_mean_55_games}
\end{figure}

We also include videos of the moves and approximate return distributions learnt by MMDQN in the supplementary. The exact video addresses in the supplementary are shown in Table \ref{tab:vid_url}. 

\begin{table}[h]
    \centering
    \setlength{\extrarowheight}{4pt}
    \begin{tabular}{l|l}
         \textbf{Games} & \textbf{Video address}   \\
         \hline 
         Breakout & \url{https://youtu.be/7P4oeJWJ6oE } \\ %\href{https://www.dropbox.com/s/sz9fds0b3rlpawg/breakout.mp4?dl=0}{\underline{Link}}  \\
         BeamRider & \url{https://youtu.be/e6VQTynnbR8} \\ %\href{https://www.dropbox.com/s/9gv084wztm2yixt/beamrider.mp4?dl=0}{\underline{Link}}\\ 
         BattleZone & \url{https://youtu.be/eXLs2pZJPCk} \\ %\href{https://www.dropbox.com/s/5q1qy96o1zvld34/battlezone.mp4?dl=0}{\underline{Link}}\\ 
         Qbert &  \url{https://youtu.be/64uHpoAPIvM} \\ %\href{https://www.dropbox.com/s/9qmztb77us9j8op/qbert.mp4?dl=0}{\underline{Link}} \\ 
         Pong & \url{https://youtu.be/NX5kXT59oJ4} \\ %\href{https://www.dropbox.com/s/yqgw0drc99k99td/pong.mp4?dl=0}{\underline{Link}}\\ 
    \end{tabular}
    \caption{Addresses to videos of the moves and approximate return distributions learnt by MMDQN.}
    \label{tab:vid_url}
\end{table}

% \begin{table}[h]
%     \centering
%     \setlength{\extrarowheight}{4pt}
%     \begin{tabular}{l|l}
%          \textbf{Games} & \textbf{Video address}   \\
%          \hline 
%          Breakout & \url{/videos/breakout.mp4} \\ %\href{https://www.dropbox.com/s/sz9fds0b3rlpawg/breakout.mp4?dl=0}{\underline{Link}}  \\
%          BeamRider & \url{/videos/beamrider.mp4} \\ %\href{https://www.dropbox.com/s/9gv084wztm2yixt/beamrider.mp4?dl=0}{\underline{Link}}\\ 
%          BattleZone & \url{/videos/battlezone.mp4} \\ %\href{https://www.dropbox.com/s/5q1qy96o1zvld34/battlezone.mp4?dl=0}{\underline{Link}}\\ 
%          Qbert & \url{/videos/qbeart.mp4} \\ %\href{https://www.dropbox.com/s/9qmztb77us9j8op/qbert.mp4?dl=0}{\underline{Link}} \\ 
%          Pong & \url{/videos/pong.mp4} \\ %\href{https://www.dropbox.com/s/yqgw0drc99k99td/pong.mp4?dl=0}{\underline{Link}}\\ 
%     \end{tabular}
%     \caption{Addresses to videos of the moves and approximate return distributions learnt by MMDQN.}
%     \label{tab:vid_url}
% \end{table}

We show the median and mean of the test human-normalized scores across all $55$ Atari games in Figure \ref{fig:med_mean_55_games}, the online learning curves in all the 55 Atari games in Figure \ref{fig:learing_curve_55_games} and provide the full raw scores of MMDQN in Table \ref{tab:raw_score_full}.

\begin{figure*}[htp]
    \centering
    \includegraphics[scale=0.35]{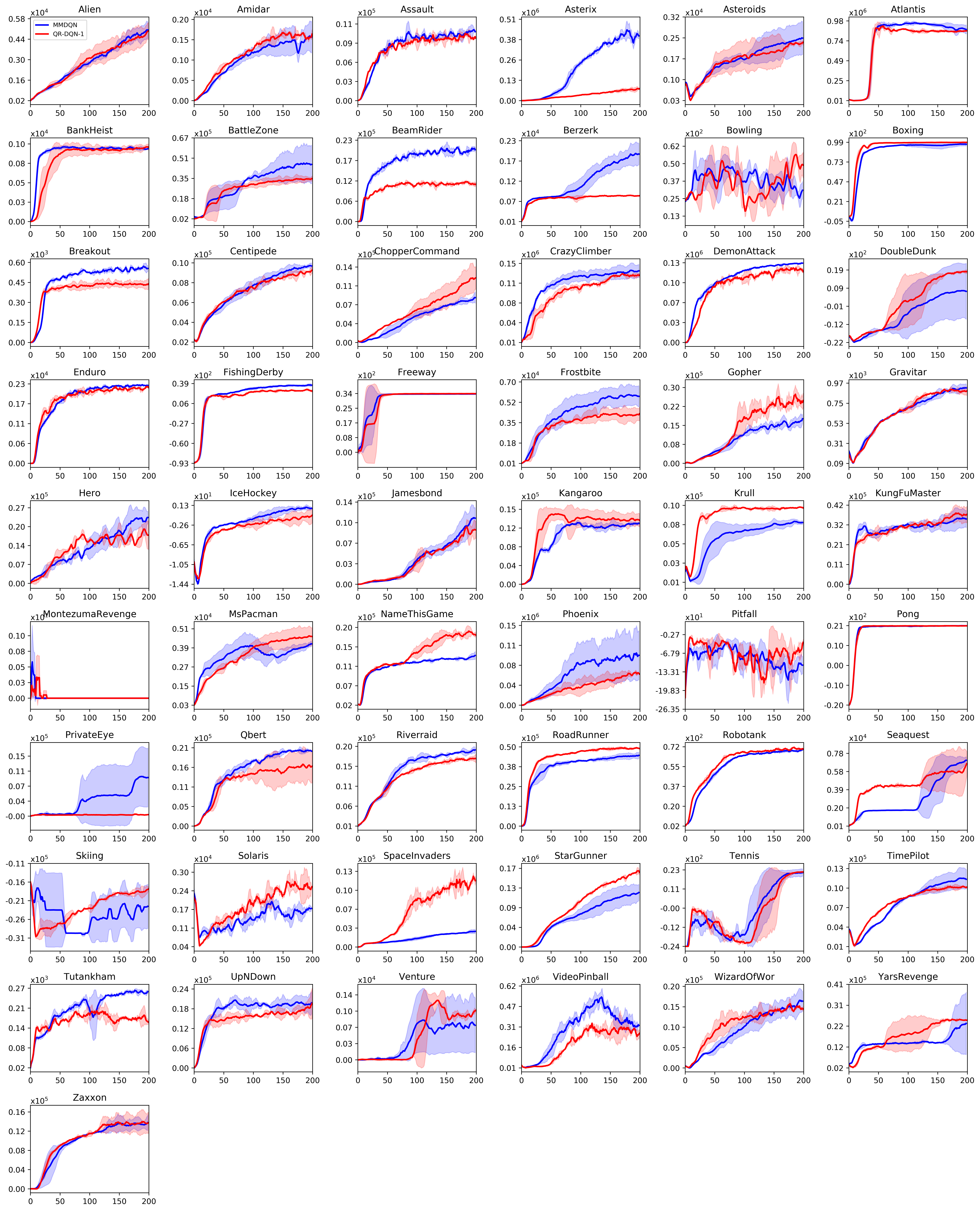}
    \caption{Online training curves for MMDQN (3 seeds) and QR-DQN-1 (2 seeds) on all 55 Atari 2600 games. Curves are averaged over the seeds and smoothed over a sliding window of 5 iterations. 95\% C.I. Reference values are from \cite{DBLP:conf/aaai/DabneyRBM18}.}
    \label{fig:learing_curve_55_games}
\end{figure*}

\begin{table*}[htp]
    \centering
    \setlength{\extrarowheight}{2.5pt}
    \begin{adjustbox}{width=0.9\textwidth}
    \begin{tabular}{l|r|r|r|r|r|r|r}
    % \resizebox{\textwidth}{!}{\begin{tabular}{l|r|r|r|r|r|r|r}
    % & 
    % & 
    % & 
    % & 
    % \multirow{3}{*}{\textbf{PRIOR.}} & 
    % & 
    % & \\
    % & 
    % & 
    % & 
    % & 
    % & 
    % & 
    % & \\
    \textbf{GAMES} & 
    \textbf{RANDOM} & 
    \textbf{HUMAN} & 
    \textbf{DQN} & 
    \textbf{PRIOR. DUEL.} & 
    \textbf{C51} & 
    \textbf{QR-DQN-1} & 
    \textbf{MMDQN} \\
    \hline
        Alien & 227.8 & 7,127.7 & 1,620.0 & 3,941.0 & 3,166 & 4,871 & \textbf{6,918.8}\\
        Amidar & 5.8 & 1,719.5 & 978.0 & 2,296.8 & 1,735 & 1,641 & \textbf{2,370.1}\\
        Assault & 222.4 & 742.0 & 4,280.4 & 11,477.0 & 7,203 & \textbf{22,012} & 19,804.7\\
        Asterix & 210.0 & 8,503.3 & 4,359.0 & 375,080.0 & 406,211 & 261,025 & \textbf{775,250.9}\\
        Asteroids & 719.1 & 47,388.7 & 1.364.5 & 1,192.7 & 1,516 & \textbf{4,226} & 3,321.3 \\ 
        Atlantis & 12,850.0 & 29,028.1 & 279,987.0 & 841,075 & 395,762.0 & 971,850 & \textbf{1,017,813.3}\\ 
        BlankHeist & 14.2 & 753.1 & 455.0 & \textbf{1,503.1} & 976 & 1,249 & 1,326.6\\ 
        BattleZone & 2,360.0 & 37,187.5 & 29,900.0 & 35,520.0 & 28,742 & 39,268 & \textbf{64,839.8}\\ 
        BeamRider & 363.9 & 16,926.5 & 8,627.5 & 30,276.5 & 14,074 & \textbf{34,821} & 34,396.2 \\
        Berzerk &  123.7 & 2,630.4 & 585.6 & \textbf{3,409.0} & 1,645 & 3,117 & 2,946.1\\ 
        Bowling & 23.1 & 160.7 & 50.4 & 46.7 & \textbf{81.8} & 77.2 & 65.8\\ 
        Boxing & 0.1 & 12.1 & 88.0 & 98.9 & 97.8 & \textbf{99.9} & 99.2 \\
        Breakout & 1.7 & 30.5 & 385.5 & 366.0 & 748 & 742 & \textbf{823.1} \\
        Centipede & 2,090.9 & 12,017.0 & 4,657.7 & 7,687.5 & 9,646 & 12,447 & \textbf{13,180.9}\\
        ChopperCommand & 811.0 & 7,387.8 & 6,126.0 & 13.185.0 & {15,600} & 14,667 & \textbf{15,687.9} \\ 
        CrazyClimber & 10,780.5 & 35,829.4 & 110,763.0 & 162,224.0 & \textbf{179,877} & 161,196 & 169,462.0 \\ 
        DemonAttack & 152.1 & 1,971.0 & 12,149.4 & 72,878.6 & 130,955 & 121,551 & \textbf{135,588.7} \\
        DoubleDunk & -18.6 & -16.4 & -6.6 & -12.5 & 2.5 & \textbf{21.9} & {12.6} \\
        Enduro & 0.0 & 860.5 & 729.0 & 2,306.4 & \textbf{3,454} & 2,355 & {2,358.5} \\
        FishingDerby & -91.7 & -38.7 & -4.9 & 41.3 & 8.9 & 39.7 & \textbf{49.6} \\ 
        Freeway & 0.0 & 29.6 & 30.8 & 33.0 & 33.9 & \textbf{34} & 33.7 \\
        Frostbite & 65.2 & 4,334.7 & 797.4 & {7,413.0} & 3,965 & 4,384 & \textbf{8,251.4} \\ 
        Gopher & 257.6 & 2,412.5 & 8,777.4 & 104,368.2 & 33,641 & \textbf{113,585} & 38,448.1\\
        Gravitar & 173.0 & 3,351.4 & 473.0 & 238.0 & 440 & 995 & \textbf{1,092.5}\\
        Hero & 1,027.0 & 30,826.4 & 20,437.8 & 21,036.5 & \textbf{38,874} & 21,395 & 28,830.7\\
        IceHockey & -11.2 & 0.9 & -1.9 & -0.4 & -3.5 & -1.7 & \textbf{3.3} \\
        JamesBond & 29.0 & 302.8 & 768.5 & 812.0 & 1,909 & 4,703 & \textbf{16,028.9} \\
        Kangaroo &  52.0 & 3,035.0 & 7,259.0 & 1,792.0 & 12,853 & \textbf{15,356} & 15,154.2 \\
        Krull & 1,598.0 & 2,665.5 & 8,422.3 & 10,374.4 & 9,735 & \textbf{11,447} & 9,447.0\\
        KungFuMaster & 258.5 & 22,736.3 & 26,059.0 & 48,375.0 & 48,192 & \textbf{76,642} & 51,011.3\\
        MontezumaRevenge & 0.0 & 4,753.3 & 0.0 & 0.0 & 0.0 & 0.0 & 0.0 \\
        MsPacman &  307.3 & 6,951.6 & 3,085.6 & 3,327.3 & 3,415 & 5,821 & \textbf{6,762.8}\\
        NameThisGame & 2,292.3 & 8,049.0 & 8,207.8 & 15,572.5 & 12,542 & \textbf{21,890} & 15,221.2  \\
        Phoenix & 761.4 & 7,242.6 & 8,485.2 & 70,324.3 & 17,490 & 16,585 & \textbf{325,395.5} \\
        Pitfall & -229.4 & 6,463.7 & -286.1 & 0.0 & 0.0 & 0.0 & 0.0 \\
        Pong & -20.7 & 14.6 & 19.5 & 20.9 & 20.9 & 21.0 & 21.0 \\
        PrivateEye & 24.9 & 69,571.3 & 146.7 & 206.0 & \textbf{15,095} & 350 & {11,366.4}\\
        QBert & 163.9 & 13,455.0 & 13,117.3 & 18,760.3 & 23,784 & \textbf{572,510} & 28,448.0 \\
        Riverraid & 1,338.5 & 17,118.0 & 7,377.6 & 20,607.6 & 17,322 & 17,571 & \textbf{23000.0}\\
        RoadRunner & 11.5 & 7,845.0 & 39,544.0 & 62,151.0 & 55,839 &\textbf{ 64,262} & 54,606.8\\
        Robotank & 2.2 & 11.9 & 63.9 & 27.5 & 52.3 & 59.4 & \textbf{74.8} \\
        Seaquest & 68.4 & 42,054.7 & 5,860.6 & 931.6 & 266,434 & \textbf{8,268} & {7,979.3} \\
        Skiing & -17,098.1 & -4,336.9 & -13,062.3 & -19,949.9 & -13,901 & \textbf{-9,324} & {-9,425.3} \\
        Solaris & 1,236.3 & 12,326.7 & 3,482.8 & 133.4 & \textbf{8,342 }& 6,740 & 4,416.5 \\
        SpaceInvaders & 148.0 & 1,668.7 & 1,692.3 & 15,311.5 & 5,747 & \textbf{20,972} & 4,387.6 \\
        StarGunner & 664.0 & 10,250.0 & 54,282.0 & 125,117.0 & 49,095 & 77,495 & \textbf{144,983.7} \\
        Tennis & -23.8 & -8.3 & 12.2 & 0.0 & 23.1 & \textbf{23.6} & 23.0 \\
        TimePilot & 3,568.0 & 5,229.2 & 4,870.0 & 7,553.0 & 8,329 & 10,345 & \textbf{14,925.3} \\
        Tutankham & 11.4 & 167.6 & 68.1 & 245.9 & 280 & 297 & \textbf{319.4} \\
        UpNDown & 533.4 & 11,693.2 & 9,989.9 & 33,879.1 & 15,612 & \textbf{71,260} & 55,309.9 \\
        Venture & 0.0 & 1,187.5 & 163.0 & 48.0 & \textbf{1,520} & 43.9 & {1,116.6} \\
        VideoPinball & 16,256.9 & 17,667.9 & 196,760.4 & 479,197.0 & \textbf{949,604} & 705,662 & 756,101.8 \\
        WizardOfWor & 563.5 & 4,756.5 & 2,704.0 & 12,352.0 & 9,300 & 25,061 & \textbf{31,446.9}\\
        YarsRevenge & 3,092.9 & 54,576.9 & 18,098.9 & \textbf{69,618.1} & 35,050 & 26,447 & 28,745.7\\
        Zaxxon & 32.5 & 9,173.3 & 5,363.0 & 13,886.0 & 10,513 & 13,112 & \textbf{17,237.9} \\
    \end{tabular}
    \end{adjustbox}
    \caption{Raw scores of MMDQN  (averaged over 3 seeds) across all 55 Atari games starting with 30 no-op actions. Reference values are from  \cite{DBLP:conf/aaai/DabneyRBM18}. }
    % \cite{DBLP:conf/icml/WangSHHLF16}, \cite{DBLP:conf/icml/BellemareDM17}, and
    \label{tab:raw_score_full}
\end{table*}

\end{document}